\documentclass[a4paper, 10pt]{article}

\usepackage{fullpage}
\usepackage[american]{babel}

\usepackage{natbib} 
\usepackage{authblk}
\usepackage{enumitem}

\usepackage{color}
\usepackage{microtype}
\usepackage{graphicx}
\usepackage{booktabs} 
\usepackage{Definitions}
\usepackage{subfig}
\usepackage{hyperref}
\usepackage{algorithm}
\usepackage[noend]{algpseudocode}
\usepackage{stackengine}

\newcommand{\xhdr}[1]{\vspace{1mm} \noindent{{\bf #1.}}}

\newcommand{\emin}{\Lambda_{\min}}
\newcommand{\emax}{\Lambda_{\max}}
\newcommand{\sdgeq}{\succcurlyeq}
\newcommand{\hes}{\nabla_{\theta} ^2\,}
\newcommand{\set}[1]{\{#1\}}

\newcommand{\hsample}{h}
\renewcommand{\mac}{m}

\newcommand{\given}{\,|\,}
\newcommand{\our}{\textsc{Differentiable Triage}{}}
\newcommand{\sdleq}{\preccurlyeq}
\newcommand{\bnm}[1]{\left\| #1 \right\|}

\begin{document}

\title{Differentiable Learning Under Triage}

\author[1]{Nastaran Okati}
\author[2]{Abir De}
\author[1]{Manuel Gomez-Rodriguez}
\affil[1]{%
{MPI for Software Systems, \{nastaran, manuelgr\}@mpi-sws.org}
}
\affil[2]{%
  {IIT Bombay, abir@cse.iitb.ac.in}
}
\date{}
\maketitle

\begin{abstract}
Multiple lines of evidence suggest that predictive models may benefit from algorithmic triage. Under algorithmic 
triage, a predictive model does not predict all instances but instead defers some of them to human experts. 
However, the interplay between the prediction accuracy of the model and the human experts under algorithmic 
triage is not well understood. In this work, we start by formally characterizing under which circumstances a predictive 
model may benefit from algorithmic triage. In doing so, we also demonstrate that models trained for full automation 
may be suboptimal under triage. Then, given any model and desired level of triage, we show that the optimal triage 
po\-li\-cy is a deterministic threshold rule in which triage decisions are derived deterministically by thresholding the 
difference between the model and human errors on a per-instance level. Building upon these results, we introduce 
a practical gradient-based algorithm that is guaranteed to find a sequence of predictive models and triage policies
of increasing performance. Experiments on a wide variety of supervised lear\-ning tasks using synthetic and real data 
from two im\-por\-tant applications---content moderation and scientific discovery---illustrate our theoretical results and 
show that the models and triage policies provided by our algorithm outperform those provided by several 
competitive baselines.
\end{abstract}

\section{Introduction} 
In recent years, there have been a raising interest on a new learning paradigm which seeks the de\-ve\-lop\-ment of predictive models 
that operate under different automation levels---models that take decisions for a given fraction of instances and leave the remaining 
ones to human experts.
This new paradigm has been so far referred to as learning under algorithmic triage~\citep{raghu2019algorithmic}, learning under human 
assistance~\citep{de2020aaai,de2021aaai}, learning to complement humans~\citep{wilder2020learning, bansal2020optimizing},
and learning to defer to an expert~\citep{sontag2020}.
Here, one does not only has to find a predictive model but also a triage policy which determines who predicts each instance.


%
The motivation that underpins learning under algorithmic triage is the observation that, while there are high-stake tasks where predictive models 
have matched, or even surpassed, the average performance of human experts~\citep{pradel2018deepbugs, topol2019high}, 
they are still less accurate than human experts on some instances, where they make far more errors than average~\citep{raghu2019algorithmic}.
The main promise is that, by working together, human experts and predictive models are likely to achieve a considerably better performance 
than each of them would achieve on their own.
While the above mentioned work has shown some success at fulfilling this promise, the interplay between the predictive accuracy of a predictive model
and its human counterpart under algorithmic triage is not well understood. 

%
%

One of the main challenges in learning under algorithmic triage is that, for each potential triage policy, there is an optimal predictive model, however, the 
triage policy is also something one seeks to optimize, as first noted by~\cite{de2020aaai}.
In this context, previous work on learning under algorithmic triage can be naturally differentiated into two lines of work.
The first line of work has developed rather general heuristic algorithms that do not enjoy theoretical guarantees~\citep{raghu2019algorithmic, bansal2020optimizing, wilder2020learning}.
The second has developed algorithms with theoretical guarantees~\citep{de2020aaai, de2021aaai, sontag2020}, however, they have focused on more 
restrictive settings. 
More specifically,~\cite{de2020aaai, de2021aaai} focus on ridge regression and support vector machines and reduce the problem to the 
maximization of approximately submodular functions and~\cite{sontag2020} view the problem from the perspective of 
cost sensitive learning and introduce a convex and consistent surrogate loss of the objective function they consider.

\xhdr{Our contributions} 
Our starting point is a theoretical investigation of the interplay between the prediction accuracy of supervised learning models and human experts 
under algorithmic triage.
By doing so, we hope to better inform the design of general purpose techniques for training differentiable models under algorithmic triage.
%
%
%
%
%
Our investigation yields the following insights: 
\begin{itemize}[leftmargin=0.8cm]
\item[I.] To find the optimal triage policy and predictive model, we need to take into account the amount of human expert disagreement,
or expert uncertainty, on a per-instance level.


\item[II.] We identify under which circumstances a predictive model that is optimal under full automation may be suboptimal under a desired
level of triage. 

\item[III.] Given any predictive model and desired level of triage, the optimal triage policy is a deterministic threshold rule in which 
triage decisions are derived deterministically by thresholding the difference between the model and human errors on a per-instance
level.
\end{itemize}
Building on the above insights, we introduce a practical gradient-based algorithm that finds a sequence of predictive models and triage policies
of increasing performance subject to a constraint on the maximum level of triage.
We apply our gradient-based algorithm in a wide variety of su\-per\-vised learning tasks using both synthetic and real-world data
from two important applications---content moderation and scientific discovery.
Our experiments illustrate our theoretical results and show that the models and triage policies provided by our algorithm outperform those provided
by several competitive baselines\footnote{\scriptsize Our code and data are available at \url{https://github.com/Networks-Learning/differentiable-learning-under-triage}}.

\xhdr{Further related work}
Our work is also related to the areas of learning to defer and active learning.
In learning to defer, the goal is to design machine learning models that are able to defer predictions~\citep{Madras2018PredictRI,bartlett2008classification,cortes2016learning,geifman2018bias,ramaswamy2018consistent,geifman2019selectivenet,liu2019deep,thulasidasan2019combating, ziyin2020learning}. 
Most previous work focuses on supervised learning and design classifiers that learn to defer either by considering the defer action as an 
additional label value or by training an independent classifier to decide about deferred decisions.
However, in this line of work, there are no human experts who make predictions whenever the classifiers defer them, in contrast with the 
literature on learning under algorithmic triage~\citep{raghu2019algorithmic, de2020aaai,de2021aaai, wilder2020learning, bansal2020optimizing, sontag2020}---they 
just pay a constant cost every time they defer predictions. Moreover, the classifiers are trained to predict the labels of all samples in the training set, as in full automation.
%
%
In active learning, the goal is to find which subset of samples one should label so that a model trained on these samples predicts accurately
\emph{any} sample at test time~\citep{cohn1995active,hoi2006batch, sugiyama2006active, willett2006faster, guo2008discriminative, sabato2014active, chen2017active,hashemi2019submodular}. 
In contrast, in our work, the trained model only needs to predict accurately \emph{a fraction} of samples picked by the triage policy at test time and rely 
on human experts to predict the remaining samples.
\section{Supervised Learning under Triage} 
\label{sec:formulation}
Let $\Xcal \subseteq \RR^{m}$ be the feature domain, $\Ycal$ be the label domain, and assume features and labels
are sampled from a ground truth distribution $P(\xb, y) = P(\xb) P(y \given \xb)$.
Moreover, let $\hat{y} = h(\xb)$ be the label predictions provided by a human expert and assume they are sampled from
a distribution $P(h \given \xb)$, which models the disagreements amongst experts~\citep{raghu2019direct}.
Then, in supervised learning under triage, one needs to find: 
\begin{itemize}[leftmargin=0.8cm]
\item[(i)] a triage policy $\pi(\xb) \,:\, \Xcal \rightarrow \{ 0, 1 \}$, which determines who predicts each feature vector---a supervised learning 
model ($\pi(\xb) = 0$) or a human expert ($\pi(\xb) = 1$);
\item[(ii)] a predictive model $m(\xb) \,:\, \Xcal \rightarrow \Ycal$, which needs to provide label predictions $\hat{y} = m(\xb)$ for those feature vectors $\xb$
for which $\pi(\xb) = 0$.
\end{itemize}
Here, similarly as in standard supervised learning, we look for the triage policy and the predictive model that result into the most accurate label predictions by 
minimizing a loss function $\ell(\hat{y}, y)$. More formally, let $\Pi$ be the set of all triage policies, then, given a hypothesis class of predictive models $\Mcal$, 
our goal is to solve the following minimization problem\footnote{\scriptsize One might think that minimizing over the set of all possible triage policies is not a well-posed 
problem, \ie, the optimal triage policy is an extremely complex function. However, our theoretical analysis will reveal that the optimal triage policy does
have a simple form and its complexity depends on the considered hypothesis class of predictive models (refer to Theorem~\ref{theorem:optimal-policy}).}:
\begin{equation} \label{eq:optimization-problem}
\underset{\pi \in \Pi, m \in \Mcal}{\text{minimize}} \quad  L(\pi, m)  \qquad
\text{subject to} \quad \EE_{\xb} [\pi(\xb)] \leq b
\end{equation}
where $b$ is a given parameter that limits the level of triage, \ie, the percentage of samples human experts 
need to provide predictions for, and
\begin{equation} \label{eq:joint-loss}
 L(\pi, m)  = \EE_{\xb, y, h}\left[ (1-\pi(\xb)) \, \ell(\mac(\xb), y) + \pi(\xb) \,   \ell(h, y) \right].
\end{equation}
Here, one might think of replacing $h$ with its point estimate $\mu_h(\xb) = \argmin_{\mu_h} \EE_{h | \xb}[\ell'(h, \mu_h)]$, where $\ell'(\cdot)$ is a general loss 
function. 
However, the re\-sul\-ting objective would have a bias term, as formalized by the following proposition\footnote{\scriptsize All proofs can be found in Appendix~\ref{app:proofs}.} for the quadratic loss $\ell'(h, \mu _h) = (h - \mu _h)^2$:
\begin{proposition} \label{prop:unbiased-biased}
Let $\ell'(\hat{y}, y) =  (\hat{y}- y)^2$  and assume there exist $\xb \in \Xcal$ for which the distribution of human 
predictions $P(h \given \xb)$ is not a point mass. Then, the function 
\begin{equation*}
\overline{L}(\pi, \mac) = \EE_{\xb, y} \left[ (1-\pi(\xb)) \, \ell'(\mac(\xb), y) + \pi(\xb) \,   \ell'( \mu_h(\xb), y) \right]  
\end{equation*}
is a biased estimate of the true average loss defined in Eq.~\ref{eq:joint-loss}.
\end{proposition}
The above result implies that, to find the optimal triage policy and predictive model, we need to take into account the amount of expert 
disagreement, or expert uncertainty, on each feature vector $\xb$ rather than just an average expert prediction.


%
\section{On the Interplay Between Prediction Accuracy and Triage}
\label{sec:policy}
Let $\mac^{*}_{0}$ be the optimal predictive model under full automation, \ie, $\mac^{*}_{0} = \argmin_{\mac \in \Mcal} L(\pi_0, \mac)$,
%
%
where $\pi_{0}(\xb) = 0$ for all $\xb \in \Xcal$.
Then, the following proposition tells us that, if the predictions made by $\mac^{*}_{0}$ are less accurate than those 
by human experts on some instances, the model will always benefit from algorithmic triage:
\begin{proposition} \label{prop:full-automation}
If there is a subset $\Vcal \subset \Xcal$ of positive measure under $P$ such that 
%
\begin{equation*}
    \int_{\xb \in \Vcal} \EE_{y | \xb}\left[ \ell(\mac^{*}_{0} (\xb), y) \right] \, dP > \int_{\xb \in \Vcal} \EE_{y, h | \xb} \left[ \ell(\hsample, y) \right] \, dP,
\end{equation*}
then there exists a nontrivial triage policy $\pi \neq \pi_0$ such that $L(\pi, \mac^{*}_{0}) < L(\pi_0, \mac^{*}_{0})$.
\end{proposition}
Moreover, if we rewrite the average loss as 
 \begin{equation*}
     L(\pi, \mac) =\EE_{\xb} \big[ (1-\pi(\xb)) \, \EE_{y | \xb}[\ell(\mac(\xb), y) ] + \pi(\xb) \, \EE_{y, \hsample | \xb} \left[ \ell(h, y) \right] \big], 
 \end{equation*}
it becomes apparent that, for any model $\mac \in \Mcal$, the optimal triage policy $\pi_{\mac}^{*} = \argmin_{\pi \in \Pi} L(\pi, \mac)$ is a deterministic 
threshold rule in which triage decisions are derived by thresholding the difference between the model and human loss on a per-instance level.
Formally, we have the following Theorem:
\begin{theorem} \label{theorem:optimal-policy}
Let $\mac \in \Mcal$ be any fixed predictive model. Then, the optimal triage policy that minimize the loss $L(\pi, \mac)$ subject to a constraint
$\EE_{\xb} [\pi(\xb)] \leq b$ on the maximum level of triage is given by:
\begin{equation} \label{eq:optimal-triage-policy}
\pi_{\mac,b}^{*}(\xb) =
\begin{cases}
1 & \text{if} \,\, \EE_{y | \xb}\left[\ell(\mac(\xb), y)- \EE_{\hsample | \xb}\left[\ell(\hsample, y) \right] \right] > t_{P, b, m} \\
0 & \text{otherwise,}
\end{cases}
\end{equation}
%
%
where
%
$t_{P, b, m} 
    = \underset{\tau \ge 0}{\text{argmin}}\,\, \EE_{\xb} \big[ \tau \,b + \max \left( \EE_{y|\xb} \left[ \ell(\mac(\xb), y) 
    -  \EE_{\hsample | \xb}[\ell(\hsample, y)] \right] - \tau, \, 0\right) \big]$.
\end{theorem}
Then, if we plug in Eq.~\ref{eq:optimal-triage-policy} into Eq.~\ref{eq:optimization-problem}, we can rewrite our minimization problem as:
\begin{equation} \label{eq:optimization-problem-2}
\underset{\mac \in \Mcal}{\text{minimize}} \quad  L(\pi^{*}_{\mac,b}, \mac)
\end{equation}
where
\begin{equation} \label{eq:loss-2}
    L(\pi^{*}_{\mac,b}, \mac) = \EE_{\xb}\big[ \EE_{y | \xb}[\ell(\mac(\xb), y) ]  - \textsc{Thres}_{t_{P, b, m}}\left( \EE_{y | \xb} \left[ \ell(\mac(\xb), y) - \EE_{\hsample | \xb}[\ell(\hsample, y)] \right], 0 \right) \big] 
\end{equation}
with
%
$\textsc{Thres}_{t}(x,\text{val}) = 
    \begin{cases}
x & \text{if} \,\, x > t \\
\text{val} & \,\, \text{otherwise.}
\end{cases}$

Here, note that, in the unconstrained case, $t_{P, b, m} = 0$ and $\textsc{Thres}_{0}(x, 0) = \max(x, 0)$.
%

%
Next, building on the above expression, we prove that the optimal predictive model under full automation $m_{\theta^{*}_0}$ is suboptimal under 
algorithmic triage as long as the average gradient across the subset of samples which are assigned to the human under the corresponding optimal 
triage policy $\pi^{*}_{m_{\theta^{*}_0}, b}$ is not zero.
%
Formally, our main result is the following Proposition:
\begin{proposition} \label{prop:suboptimality}
Let $m_{\theta^{*}_0}$ be the optimal predictive model under full automation within a hypothesis class of 
parameterized models $\Mcal(\Theta)$, $\pi^{*}_{m_{\theta^{*}_0}, b}$ the optimal triage policy for $m_{\theta^{*}_0}$ 
defined in Eq.~\ref{eq:optimal-triage-policy} for a given maximum level of triage $b$, and 
$\Vcal = \{ \xb \given \pi^{*}_{m_{\theta^{*}_0,b}}(\xb) = 1 \}$. If
\begin{equation} \label{eq:gradient-condition}
    \int_{\xb \in \Vcal} \EE_{y|\xb}\left[ \left. \nabla_{\theta} \ell(\mac_{\theta}(\xb), y) \right|_{\theta = \theta^{*}_0} \right] \, dP \neq \bm{0}.
\end{equation}
then it holds that $L(\pi^{*}_{m_{\theta^{*}_0},b}, m_{\theta^{*}_0}) > \min_{\theta \in \Theta} L(\pi^{*}_{\mac_{\theta},b}, \mac_{\theta})$.
\end{proposition}
Finally, we can also identify the circumstances under which any predictive model $m_{\theta'}$ within a hypothesis class of parameterized 
predictive models $\Mcal(\Theta)$ is suboptimal under algorithmic triage:
\begin{proposition} \label{prop:suboptimality-2}
Let $m_{\theta'}$ be a predictive model within a hypothesis class of parameterized models $\Mcal(\Theta)$, $\pi^{*}_{m_{\theta'},b}$ the optimal 
triage policy for $m_{\theta'}$ defined in Eq.~\ref{eq:optimal-triage-policy} for a given maximum level of triage $b$, and 
$\Vcal = \{ \xb \given \pi^{*}_{m_{\theta'},b}(\xb) = 0 \}$. If
\begin{equation} \label{eq:gradient-condition}
    \int_{\xb \in \Vcal} \EE_{y|\xb}\left[ \left. \nabla_{\theta} \ell(\mac_{\theta}(\xb), y) \right|_{\theta = \theta'} \right] \, dP \neq \bm{0}.
\end{equation}
then it holds that $L(\pi^{*}_{m_{\theta'},b}, m_{\theta'}) > \min_{\theta \in \Theta} L(\pi^{*}_{\mac_{\theta},b}, \mac_{\theta})$.
\end{proposition}
The above results will lay the foundations for our practical gradient-based algorithm for differentiable learning under triage in the next
section.

\section{How To Learn Under Triage} 
\label{sec:gradient}
In this section, our goal is to find the policy $m_{\theta^{*}}$ within a hypothesis class of parameterized predictive models $\Mcal(\Theta)$ 
that minimizes the loss $L(\pi^{*}_{m_{\theta},b}, m_{\theta})$ defined in Eq.~\ref{eq:optimization-problem-2}.

To this end, we now introduce a general purpose gradient-based algorithm that first approximates $m_{\theta^{*}}$ given a desirable 
maximum level of triage $b$ and then approximates the corresponding optimal triage policy $\pi^{*}_{m_{\theta^{*}},b}$\footnote{\scriptsize At
test time, given a predictive model $m_{\theta}$ and an unseen sample $\xb$, we cannot directly evaluate (or, more precisely, estimate using Monte-Carlo) 
the value of the optimal triage policy $\pi^{*}_{m_{\theta^{*}},b}(\xb)$, given by Eq.~\ref{eq:optimal-triage-policy}, since it depends on $\EE_{y \given \xb}[\cdot]$ 
and $\EE_{h \given \xb}[\cdot]$.}.
To approximate $m_{\theta^{*}}$, the main obstacle we face is that the threshold value $t_{P, b, m_{\theta}}$ in the average loss $L(\pi^{*}_{m_{\theta},b}, m_{\theta})$ 
given by Eq.~\ref{eq:loss-2} depends on the predictive model $m_{\theta}$ which we are trying to learn.  
To overcome this challenge, we proceed sequentially, starting from the triage policy $\pi_0$, with $\pi_0(\xb) = 0$ for all $\xb \in \Xcal$, and build a sequence 
%
of triage policies and predictive models $\{ (\pi^{*}_{m_{\theta_{t}},b}, m_{\theta_{t}}) \}_{t=0}^{T}$.
More specifically, in each step $t$, we find the parameters of the predictive model $m_{\theta_{t}}$ 
via stochastic gradient descent (SGD)~\citep{robbins1951stochastic}, \ie,
\begin{align}
\theta_{t}^{(j)} &= \theta_{t}^{(j-1)} - \alpha^{(j-1)} \nabla_{\theta} \left. L(\pi^{*}_{m_{\theta_{t-1}},b}, m_{\theta}) \right|_{\theta=\theta_{t}^{(j-1)}} \nn \\
&= \theta_{t}^{(j-1)} - \alpha^{(j-1)} \nabla_{\theta} \EE_{\xb} \big[ \pi^{*}_{m_{\theta_{t-1}},b}(\xb) \, \EE_{y, \hsample | \xb} \left[ \ell(h, y) \right]  \left. + (1-\pi^{*}_{m_{\theta_{t-1}},b}(\xb)) \, \EE_{y | \xb}[\ell(\mac_{\theta}(\xb), y) ] \big] \right|_{\theta = \theta_{t}^{(j-1)}}\nn \\
&= \theta_{t}^{(j-1)} - \alpha^{(j-1)} \EE_{\xb} \big[ (1-\pi^{*}_{m_{\theta_{t-1}},b}(\xb)) \times \EE_{y | \xb}[ \nabla_{\theta} \left. \ell(\mac_{\theta}(\xb), y) \right|_{\theta = \theta_{t}^{(j-1)}} ] \big], \label{eq:sgd-step-1}
\end{align}
%
where $\alpha^{(j)}$ is the learning rate at iteration $j$. 
%
Moreover, the following proposition shows that, under mild conditions, the performance of the triage 
policies and predictive models improves in each step: 

\begin{proposition}\label{prop:GD-error-change}
Assume that $\nabla_{\theta} ^2 \ell(m_{\theta}(\xb),y)$$\preccurlyeq$ $\Lambda \II$ for all $\xb \in \Xcal$ and $y \in \Ycal$ and $\alpha^{(j)} < 1/\Lambda$ for 
all $j > 0$ for some constant $\Lambda > 0$.
If, in each step $t$, we find the parameters of the predictive model $m_{\theta_t}$ using Eq.~\ref{eq:sgd-step-1}, with $\theta_{t}^{(0)} = \theta_{t-1}$,
then, it holds that $L(\pi^{*}_{m_{\theta_{t}},b}, m_{\theta_{t}} ) < L(\pi^{*}_{m_{\theta_{t-1}},b}, m_{\theta_{t-1}})  $.
\end{proposition}
%
In addition, the following theorem shows that, whenever the loss function $\ell(\cdot)$ is convex with respect to $\theta$, our algorithm enjoys global 
convergence guarantee: 
\begin{theorem}\label{thm:global}
Let $\ell(\cdot)$ be convex with respect to $\theta$ and the output of the SGD algorithm $\theta_t = \argmin_{\theta} L(\pi^* _{m_{\theta_{t-1}},b}, m_{\theta})$. Moreover, assume that $\emin\II \sdleq \hes \ell(m_{\theta}(\xb),y) \sdleq \emax\II$, with $\emin > 0$, and $\ell(\cdot)$ be $H$-Lipschitz, \ie, $\ell(m_{\theta}(\xb),y)-\ell(m_{\theta'}(\xb),y) \le H\cdot\bnm{\theta-\theta'}$. Then, we have that
 \begin{align}
   \lim_{t \to \infty} L( \pi^* _{m_{\theta_{t}},b},m_{\theta_{t}})-L( \pi^* _{m_{\theta^*},b},m_{\theta^*}) \le  \frac{4H^2 \emax}{\emin^{2} (1-b)^{2}}.
 \end{align}
\end{theorem}

%
In practice, given a set of samples $\Dcal = \{ (\xb_i, y_i, h_i) \}$, we can use the follo\-wing finite sample Monte-Carlo estimator for the gradient 
$\nabla_{\theta} L(\pi^{*}_{m_{\theta_{t-1}},b}, m_{\theta})$:
\begin{align*}
\nabla_{\theta} L(\pi^{*}_{m_{\theta_{t-1}},b}, m_{\theta}) &= \nabla_{\theta} \left[ \frac{1}{|\Dcal|} \sum_{i=1}^{|\Dcal|} \ell(m_\theta(\xb_i), y_i) - \textsc{Thres}_{t_{P, b, m_{\theta_{t-1}}}}\left( \ell(\mac_{\theta}(\xb_i), y_i) - \ell(\hsample_i, y_i), 0 \right) \right] \\
&=\frac{1}{|\Dcal|} \sum_{i=1}^{\max(\lceil (1-b) \, |\Dcal| \rceil, p)} \nabla_{\theta} \ell(m_{\theta}(\xb_{[i]}), y_{[i]})
\end{align*}
where $\cdot_{[i]}$ denotes the $i$-th sample in increasing value of the difference between the model and the human 
loss\footnote{\scriptsize Note that, if the set of samples contains several predictions $h_{[i]}$ by different human experts 
for each sample $\xb_{[i]}$, we would use all of them to estimate the (average) human loss.} $\ell(m_{\theta_{t-1}}(\xb_{[i]}), y_{[i]}) - \ell(h_{[i]}, y_{[i]})$
and $p$ is the number of samples with $\ell(m_{\theta_{t-1}}(\xb_{[i]}), y_{[i]}) - \ell(h_{[i]}, y_{[i]}) < 0$.

In the above, we do not have to explicitly compute the threshold $t_{P, b, m_{\theta_{t-1}}}$ nor the triage policy $\pi^{*}_{m_{\theta_{t-1}},b}(\xb_i)$ 
for every sample $\xb_i$ in the set $\Dcal$, we just need to pick the $\max(\lceil (1-b) \, |\Dcal| \rceil, p)$ samples with the lowest value of the model loss minus 
the human loss $ \ell(m_{\theta_{t-1}}(\xb_{[i]}), y_{[i]}) - \ell(h_{[i]}, y_{[i]})$ using the predictive model $m_{\theta_{t-1}}$ fitted in step $t-1$.
To understand why, note that, as long as $t_{P, b, m_{\theta_{t-1}}} > 0$, by definition, $t_{P,b,m_{\theta_{t-1}}}$ needs to satisfy that
\begin{equation*}
\frac{d}{d\tau} \bigg[ \sum_{i \in \Dcal} [ \tau b + \max( \ell(m_{\theta_{t-1}}(\xb_i), y_i) - \left. \ell(h_i, y_i) - \tau, 0) ] \bigg] \right|_{\tau = t_{P,b,m_{\theta_{t-1}}}} = 0
\end{equation*}
and this can only happen if 
$\ell(m_{\theta_{t-1}}(\xb_i), y_i) - \ell(h_i, y_i) - t_{P,b,m_{\theta_{t-1}}} > 0$
%
for $\lfloor b\, |\Dcal| \rfloor$ out of $|\Dcal|$ samples.
Here, we are implicitly estimating the optimal triage policies using the observed training labels and human predictions---we are not approximating them 
using a pa\-ra\-me\-te\-rized model---and, due to Proposition~\ref{prop:GD-error-change}, the implementation of the above procedure with Monte-Carlo estimates is 
guaranteed to converge to a local minimum of the empirical loss.
Moreover, note that we can think of the procedure as a particular instance of disciplined parameterized programming~\citep{amos2017optnet, agrawal2019differentiable}, 
where the differentiable convex optimization layer is given by the minimization with respect to the triage policy.

While training each of the predictive models $m_{\theta_t}$, we can implicitly compute the optimal triage policy $\pi^{*}_{\theta_{t}, b}$ as described above, however, 
at test time, we cannot do the same since we would need to observe the label and human prediction of each unseen sample $\xb$.
To overcome this, after training the last machine model $m_{\theta_{T}}$, we also fit a model $\hat{\pi}_{\gamma}(\xb)$ to approximate 
$\pi^{*}_{m_{\theta_{T}},b}(\xb)$ using SGD, \ie,
\begin{equation} \label{eq:sgd-step-gamma}
\gamma^{(j)} = \gamma^{(j-1)} - \alpha^{(j-1)} \left. \nabla_{\gamma} \left[ \sum_{i=1}^{|\Dcal|} \ell'{}(\hat{\pi}_{\gamma}(\xb_i), \pi^{*}_{m_{\theta_{T}},b}(\xb_i)) \right] \right|_{\gamma = \gamma^{(j-1)}}, 
\end{equation}
where $\alpha^{(j)}$ is the learning rate at iteration $j$ and the choice of loss $\ell'{}(\cdot)$ depends on the model class chosen for $\hat{\pi}_{\gamma}$.
%
%
%
%
In our experiments, we have found that, using the above procedure, we can approximate well the optimal triage policy $\pi_{m_{\theta_{T}}, b}$.
However, we would like to note that this problem can also be viewed as finding an estimator for the $\alpha$-superlevel set $C^{\alpha}(f) = \{ \xb \in \Xcal \,:\, f(\xb) \geq \alpha \}$ 
of the function $f(\xb) = \EE_{y \given \xb}[\ell(m_{\theta_{T}}(\xb), y) - \EE_{h \given \xb}[\ell(h(\xb), y)]]$ with $\alpha = t_{P,b,m_{\theta_{T}}}$ from a set of noisy observations.
Under this view, it might be possible to derive estimators with error performance bounds building upon recent work on level set estimation~\citep{willett2007minimax,singh2008nonparametric}.
which is left as future work. Refer to Algorithm~\ref{alg:sgd} for a pseudocode implementation of the overall gradient-based algorithm, which returns $\theta_{T}$ and $\gamma$.
Appendix~\ref{app:sgd} provides a detailed scalability analysis, which suggests that our algorithm does not significantly increase the computational complexity of vanilla SGD.
%
%
%
%
%
\begin{algorithm}[!!!t]
\caption{\textsc{\our}: it returns the weights of a predictive model $m_{\theta}$ and the weights of a triage policy $\hat{\pi}_{\gamma}$.}
\label{alg:sgd}
 \small
  \begin{algorithmic}[1]
    \Require Set of training samples $\Dcal$, maximum level of triage $b$, number of time steps $T$, number of epochs $N$, 
    mini batches $M$, batch size $B$, learning rate $\alpha$.
    \Statex
    \Function{TrainMachineUnderTriage}{$T$,$\Dcal$, $M$, $B$, $b$, $\alpha$}
    	\State $\theta^{(0)} \gets  \Call{InitializeTheta}{{}}$
       \For{$t = 1, \ldots, T$}
            \State $\theta_t \gets \Call{TrainModel}{\theta_{t-1}, \Dcal, M, B, b, \alpha}$
        \EndFor
        \State $\gamma \gets \Call{ApproximateTriagePolicy}{\theta_{T}, \Dcal, N, M, B, b, \alpha}$
        \State \Return $\theta_{T}, \gamma$
   \EndFunction
	
    \Statex
     \Function{Triage}{$\Dcal, b, \theta$}
     	\State $p \gets$  number of instances in $\Dcal$ with $\ell(m_{\theta}(\xb), y) - \ell(h, y) < 0$
      	\State $\Dcal' \gets \emptyset$
     	\For{$i = 1, \ldots, \max((1-b)|\Dcal|, p)$}
                \State $\Dcal' \gets \Dcal' \cup \{ \text{$i$-th sample from $\Dcal$ in increasing value of $\ell(m_{\theta}(\xb), y) - \ell(h, y)$} \}$
         \EndFor
%
 %
        \State \Return $\Dcal'$
    \EndFunction
   \Statex
    \Function{TrainModel}{$\theta'$, $\Dcal$, $M$, $B$, $b$, $\alpha$}
        \State $\theta^{(0)} \gets \theta'$
        \For{$i=0,\ldots, M-1$}
		\State $\Dcal^{(i)} \gets$ the i'th mini batch of $\Dcal$
		\State $\Dcal^{(i)}  \gets \Call{Triage}{\Dcal^{(i)}, b, \theta^{(i)}}$
        		\State $\nabla \gets 0$
        		\For{$(\xb, y, h) \in \Dcal^{(i)}$}
                 	 \State $\nabla \gets \nabla + \nabla_{\theta} \left. \ell(m_{\theta}(\xb), y) \right|_{\theta=\theta^{(i)}}$
         	\EndFor
         	\State $\theta^{(i+1)} \gets \theta^{(i)} - \alpha \, \frac{\nabla}{B}$
        \EndFor
        \State \Return $\theta^{(M)}$
    \EndFunction
    \Statex
    \Function{ApproximateTriagePolicy}{$\theta$, $\Dcal$, $N$, $M$, $B$, $b$, $\alpha$}
        \State $\gamma^{(M)} \gets \Call{InitializeGamma }{{}}$
        \For{$j=1,\ldots, N$}
        \State $\gamma^{(0)} \gets \gamma^{(M)}$
        \For{$i=0,\ldots, M-1$}
               \State $\Dcal^{(i)} \gets$ the i'th mini batch of $\Dcal$
        		\State $\nabla \gets 0$
        		\For{$(\xb, y, h) \in \Dcal^{(i)}$}
        			\State $\nabla \gets \nabla + \left. \nabla_{\gamma} \ell'(\hat{\pi}_{\gamma}(\xb), \pi^{*}_{m_{\theta},b}(\xb)) \right|_{\gamma=\gamma^{(i)}}$
		\EndFor
		\State $\gamma^{(i+1)} \gets \gamma^{(i)} - \alpha \, \frac{\nabla}{B}$
	\EndFor
	\EndFor
	\State \Return $\gamma^{(M)}$
    \EndFunction
  \end{algorithmic}
\end{algorithm}

\begin{figure*}[!t]
\addtolength\tabcolsep{-5pt}
\centering
\hspace{0.1cm}{\includegraphics[width=0.9\textwidth]{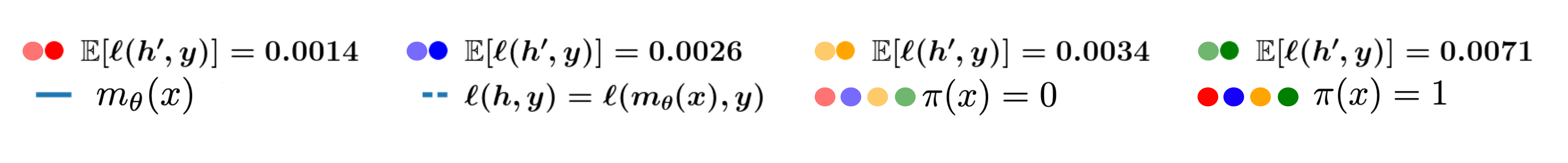}}\\[-3ex]
\subfloat[Predictive model $m_{\theta_0}$ trained under full automation]{
\begin{tabular}{cc}
    \includegraphics[width=0.24\textwidth]{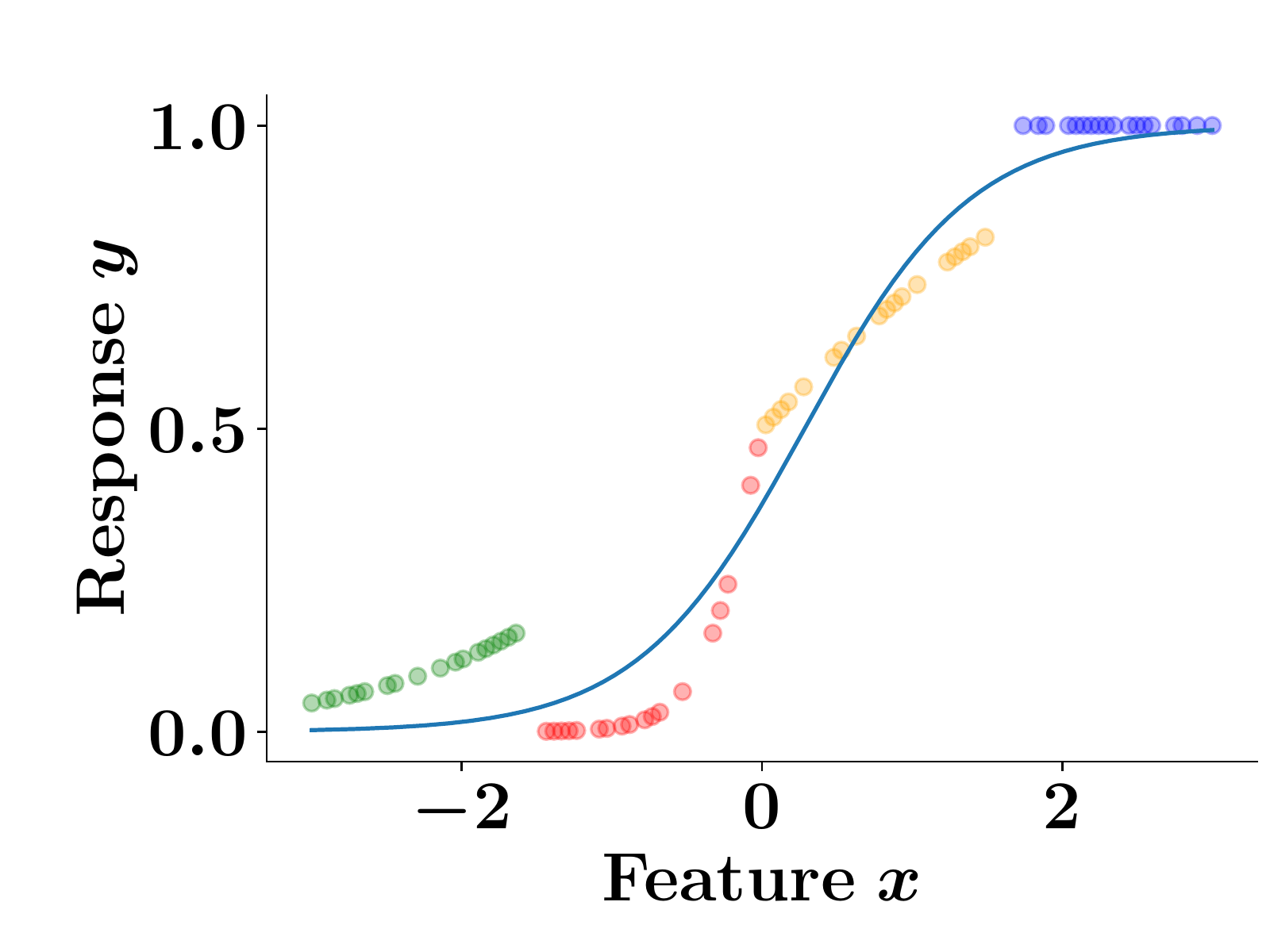} &\includegraphics[width=0.24\textwidth]{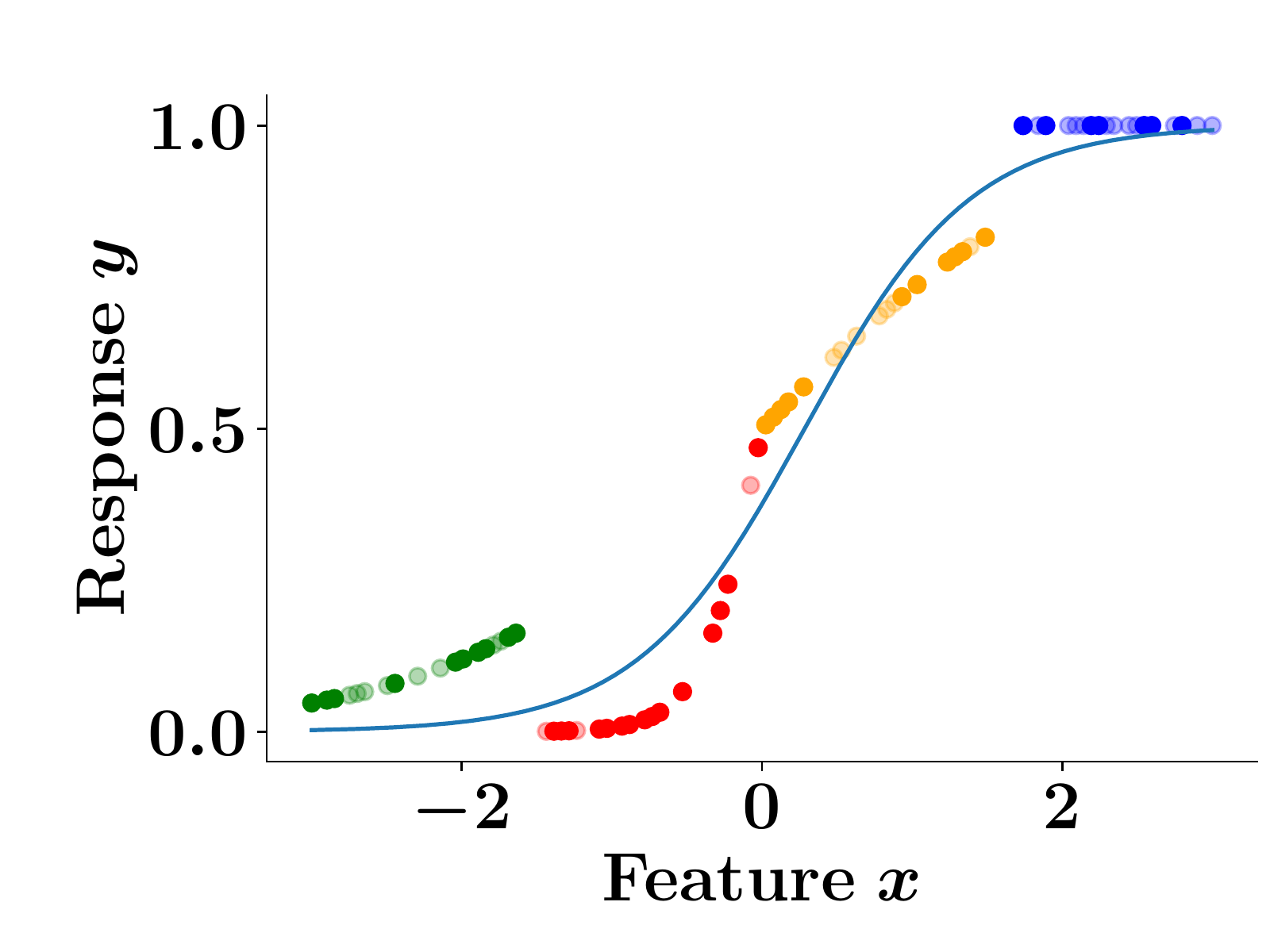} \vspace{-1mm}  \\
    \stackunder[3pt]{\includegraphics[width=0.24\textwidth]{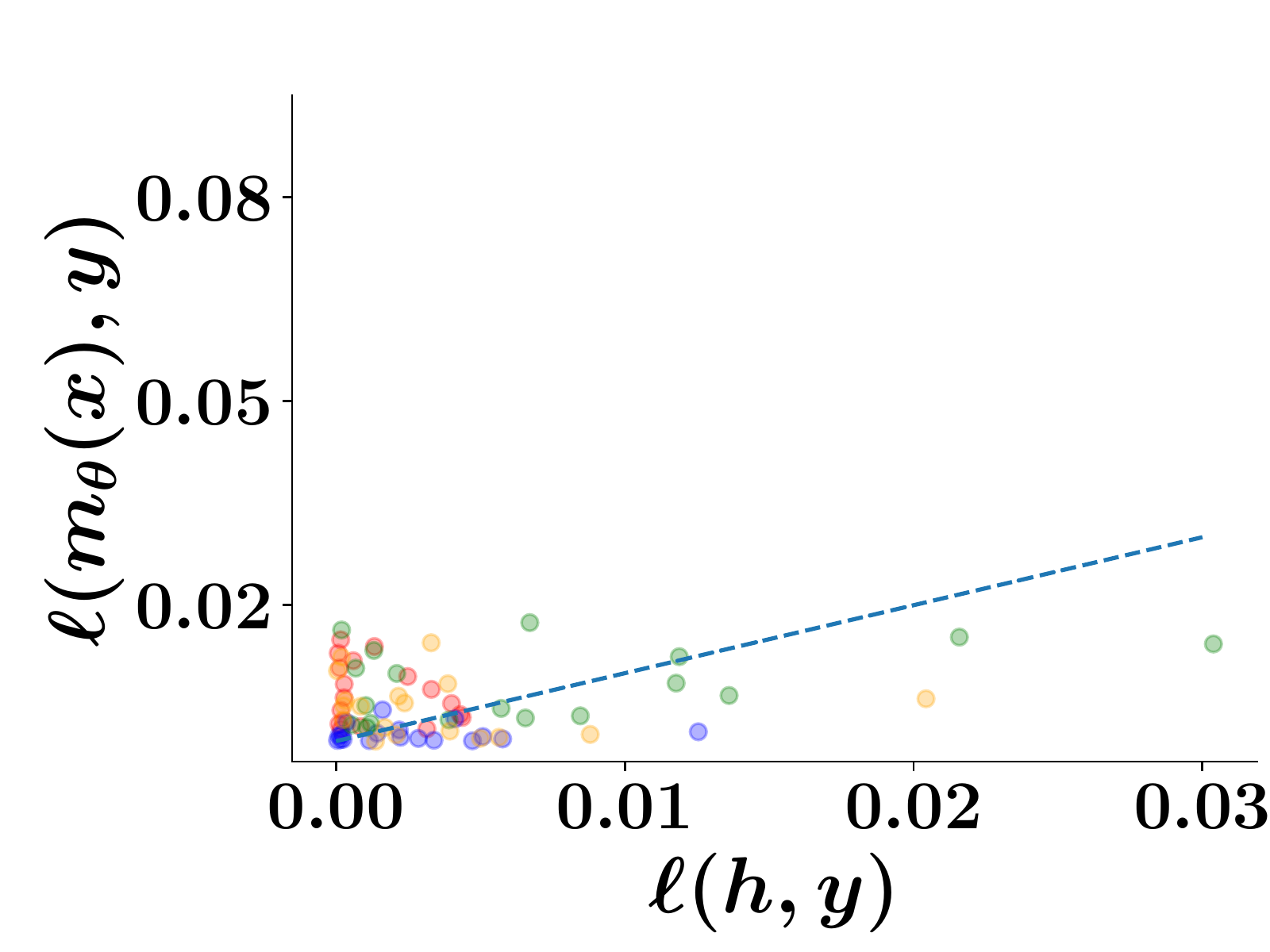}}{\scriptsize No triage} &  \stackunder[5pt]{\includegraphics[width=0.24\textwidth]{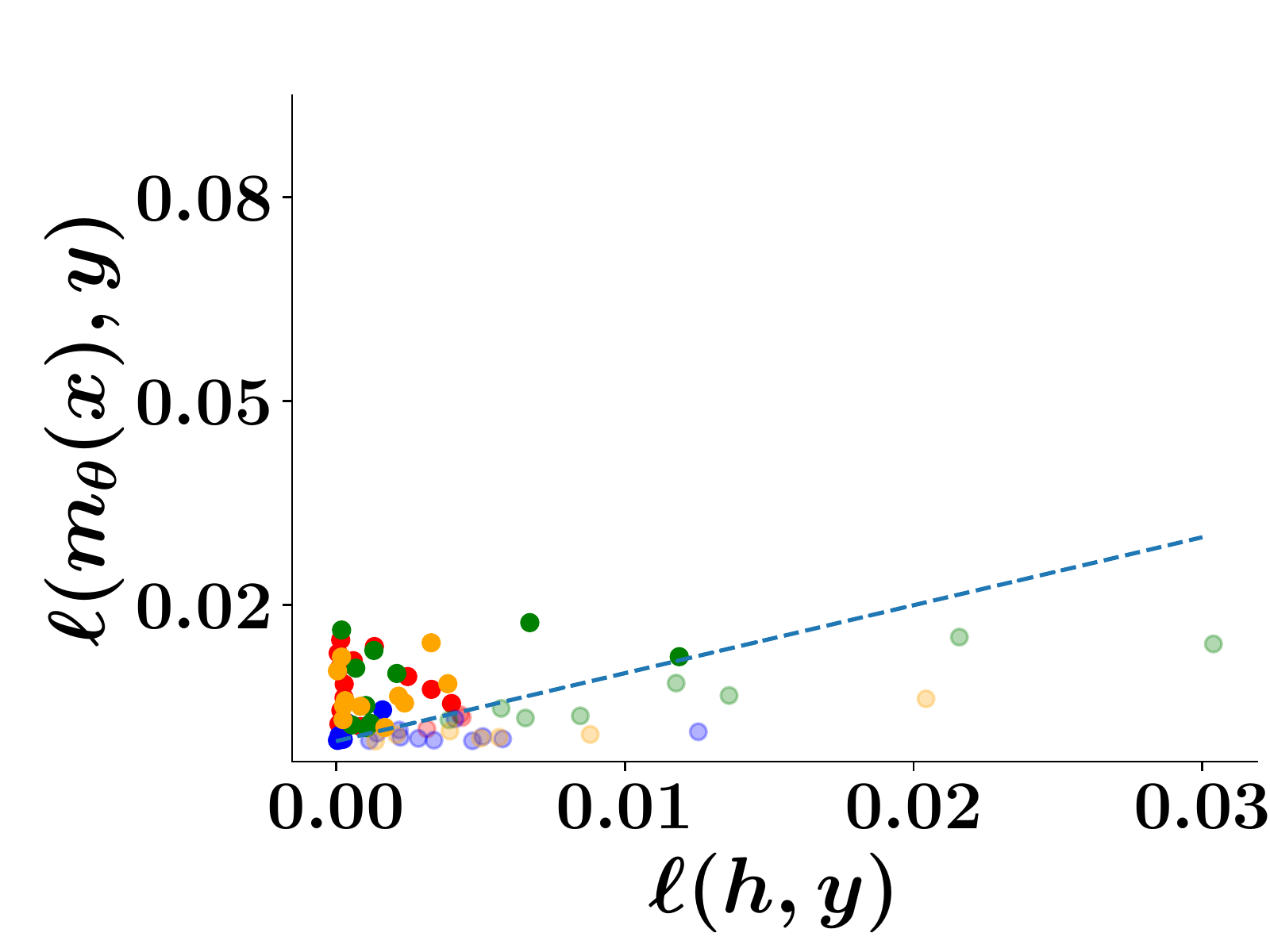}}{\scriptsize  Triage policy $\pi^{*}_{m_{\theta_0, b}}$}
 \end{tabular}
}
\subfloat[Predictive model $m_{\theta}$ trained under triage]{
\begin{tabular}{cc}
    \includegraphics[width=0.24\textwidth]{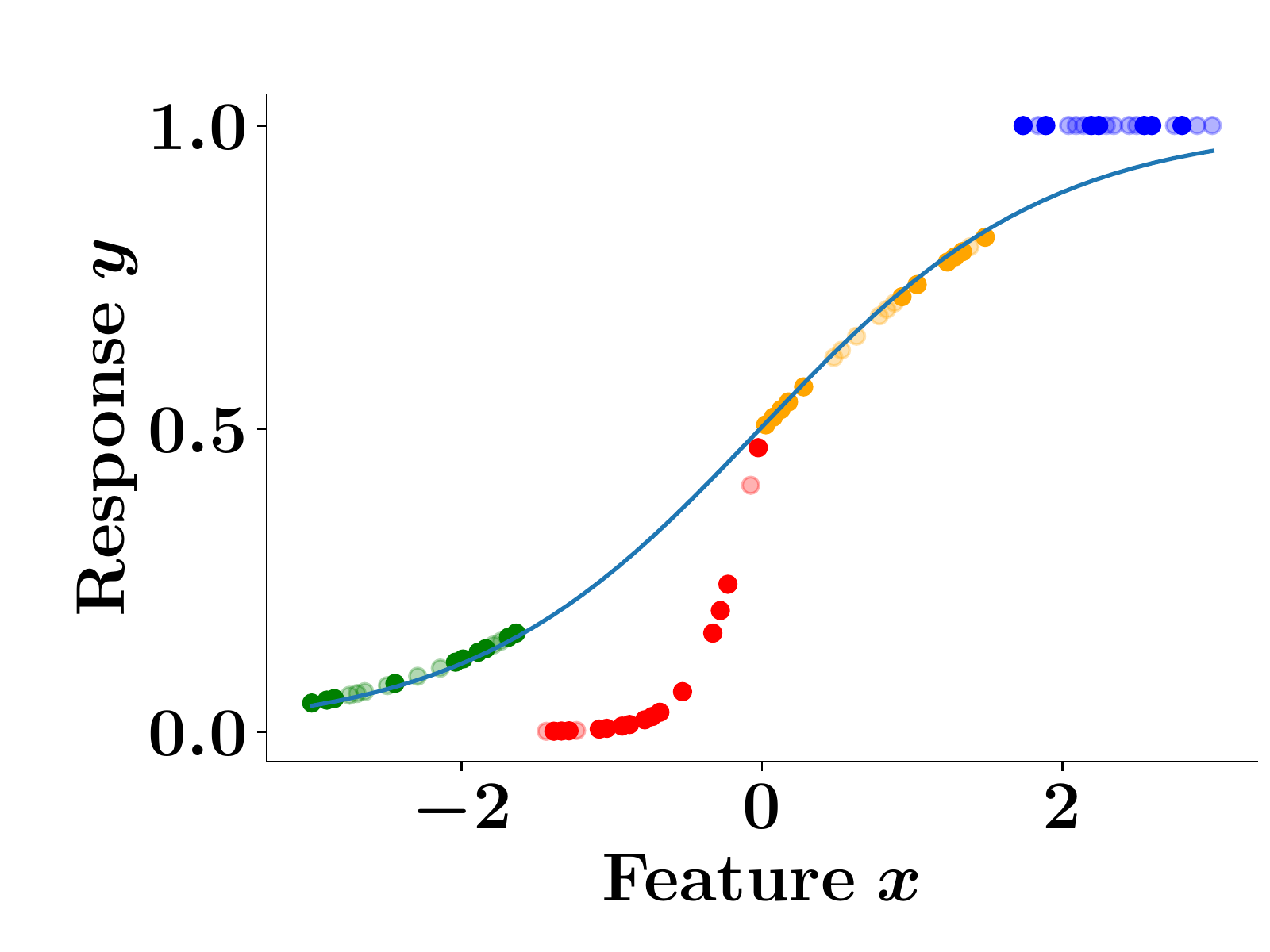} & \includegraphics[width=0.24\textwidth]{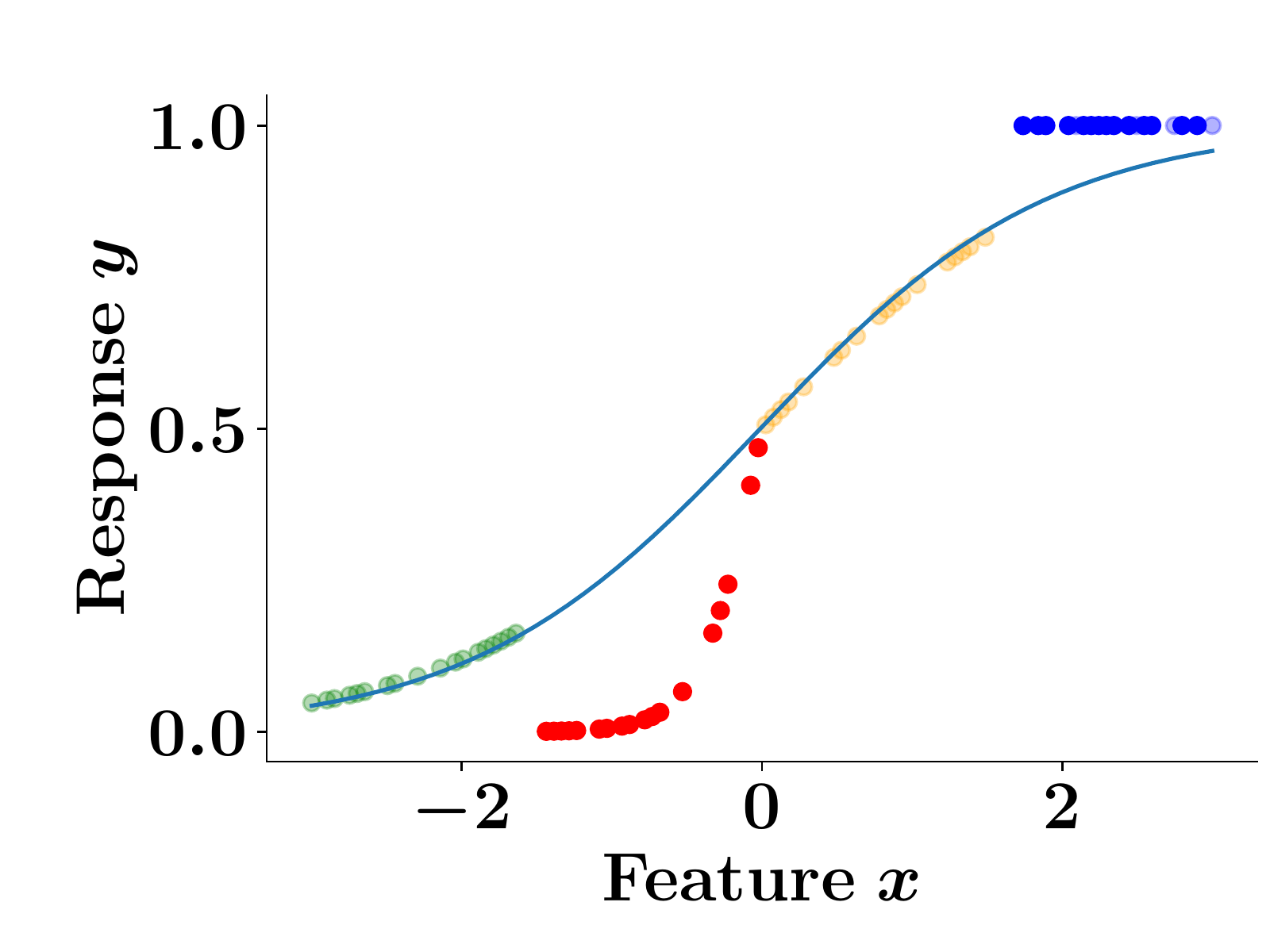} \vspace{-1mm} \\
    \stackunder[3pt]{\includegraphics[width=0.24\textwidth]{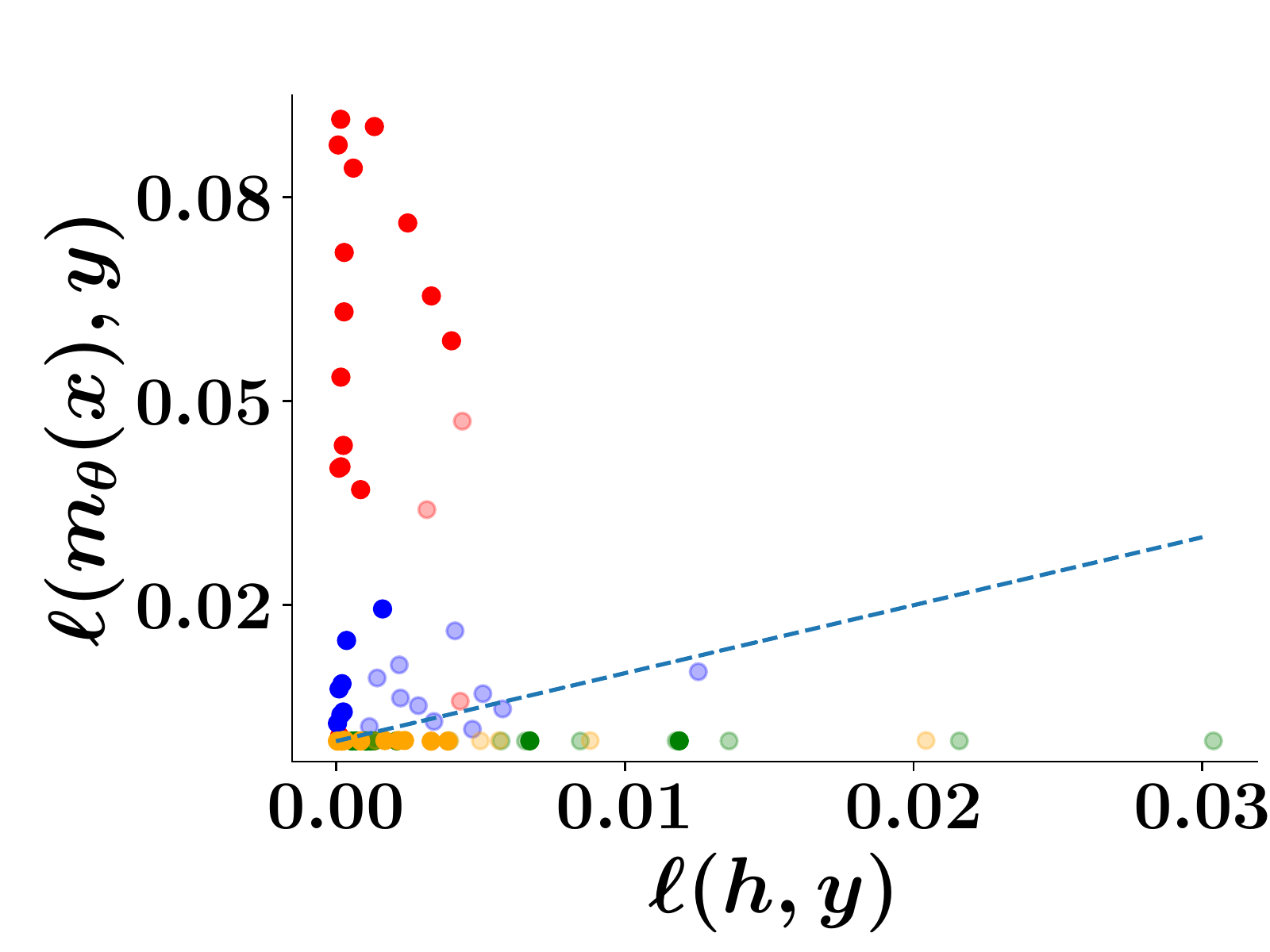}}{\scriptsize Triage policy $\pi^{*}_{m_{\theta_0}, b}$} &  \stackunder[5pt]{\includegraphics[width=0.24\textwidth]{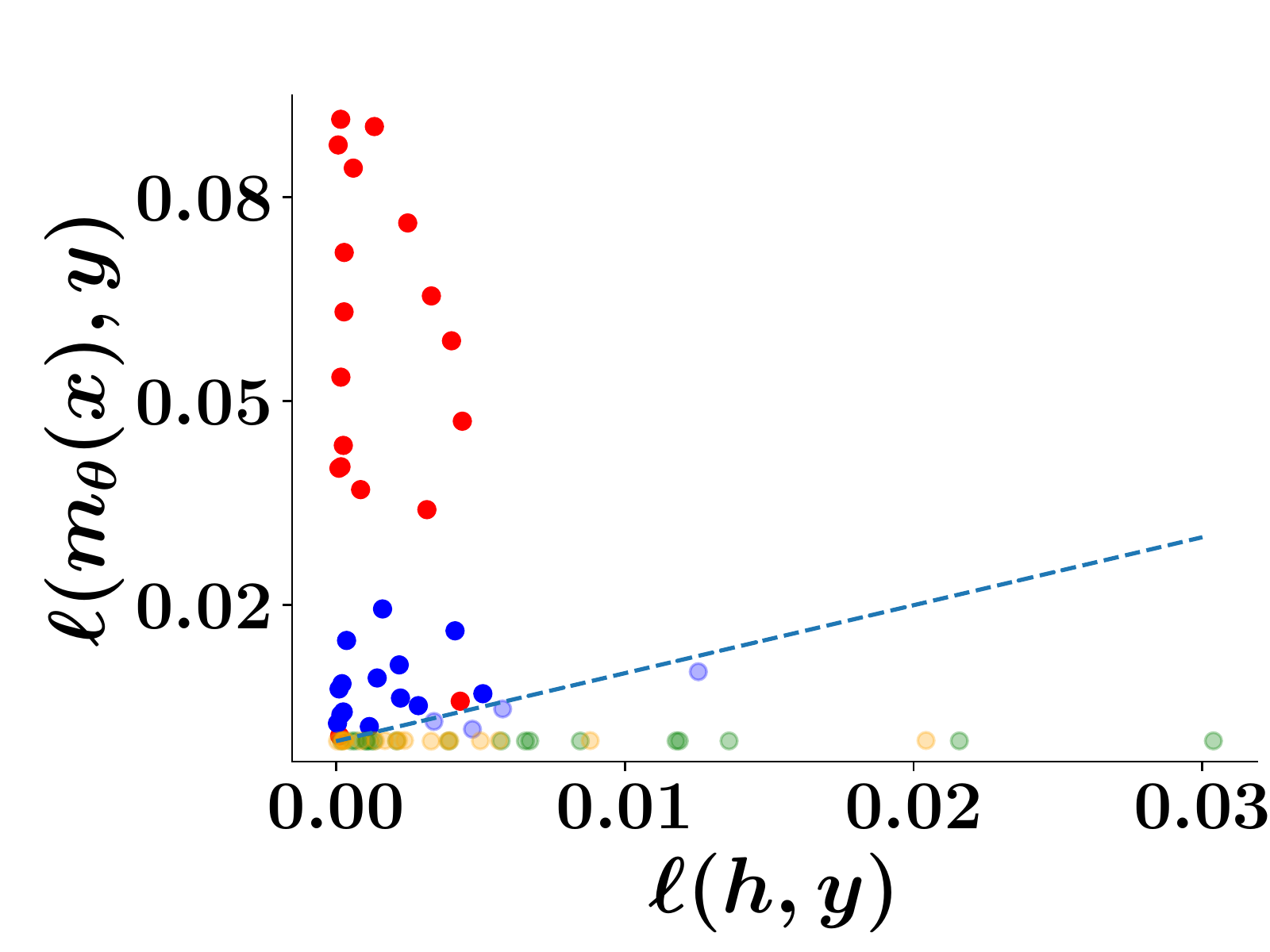}}{\scriptsize Triage policy $\pi^*_{m_{\theta}, b}$}
\end{tabular}
}
\caption{Interplay between the per-instance accuracy of predictive models and experts under different triage policies.
In both panels, the first row shows the training samples $(x, y)$ along with the predictions made by the models $m(x)$ and the triage policy 
values $\pi(x)$ and the second row shows the predictive model loss $\ell(m(x),y)$ against the human expert loss $\ell(m(x),y)$ on a per-instance level. Columns 
correspond to the settings (1--4) from left to right.
%
%
The triage policy $\pi^{*}_{m_{\theta_0},b}$ is optimal for the predictive model $m_{\theta_0}$ and the triage policy $\pi^{*}_{m_{\theta},b}$ is optimal for the predictive 
model $m_{\theta}$.
%
%
%
Each point corresponds to one instance and, for each instance, the color indicates the amount of noise in the predictions by experts, as given by Eq.~\ref{eq:noise}, 
and the tone indicates the triage policy value.
In all panels, we used $\ell(\hat{y},y)=(\hat{y}-y)^2$ and the class of predictive models parameterized by sigmoid functions, \ie, $m_{\theta}(x)=S_{\theta}(x)$.
%
%
}\vspace{-2mm}
\label{fig:first_synthetic}
\end{figure*}

\section{Experiments on Synthetic Data} 
\label{sec:synthetic}
%
In this section, our goal is to shed light on the theoretical results from Section~\ref{sec:policy}. To this end, we use our gradient-based algorithm in 
a simple regression task in which the optimal predictive model under full automation is suboptimal under algorithmic triage.\footnote{\label{infrastructure}\scriptsize All 
algorithms were implemented in Python 3.7 and ran on a \href{https://images.nvidia.com/content/technologies/volta/pdf/tesla-volta-v100-datasheet-letter-fnl-web.pdf}{V100 Nvidia Tesla GPU with 32GB of memory}.}
%

\xhdr{Experimental setup}
We generate $|\Dcal|=72$ samples, where
we first draw the features $x \in\RR$ uniformly at random, \ie,
$x \sim \text{U}[-3,3]$, and then obtain the response variables $y$ using two different sigmoid functions $S_{\theta}(x) = \frac{1}{1+\exp(- \theta x)}$.
More specifically, we set $y = S_{1}(x)$ if $x \in [-3,-1.5)\cup[0,1.5)$ and $y = S_{5}(x)$ if $x \in [-1.5,0)\cup [1.5,3]$.
Moreover, we assume human experts provide noisy predictions of the response variables, \ie, $h(x) = y + \epsilon(x)$, where 
$\epsilon(x) \sim \Ncal(0,\sigma_{\epsilon} ^2(x))$ with
\begin{align} \label{eq:noise}
 \sigma_{\epsilon} ^2 (\xb) =\left\{
\begin{array}{cc}
8\times 10^{-3}& \quad \text{if}\quad x \in [-3,-1.5)\\
1\times 10^{-3}& \quad \text{if}\quad x \in [-1.5,0)\\
4\times 10^{-3}& \quad \text{if}\quad x \in [0,1.5)\\
2\times 10^{-3}& \quad \text{if}\quad x \in [1.5,3]\\
\end{array}\right.
\end{align}
In the above, we are using heteroscedastic noise motivated by multiple lines of evidence that suggest that human experts performance on a
per instance level spans a wide range~\citep{raghu2019algorithmic, raghu2019direct, de2020aaai}.
%
%
Then, we consider the hypothesis class of predictive models $\Mcal(\Theta)$ parameterized by sigmoid functions, \ie, $m_{\theta}(x) = S_{\theta}(x)$, and utilize the sum 
of squared errors on the predictions as loss function, \ie, $\ell(\hat{y},y)=(\hat{y}-y)^2$, to train the models and triage policies in the following four settings:
%
%
%
\begin{itemize}[noitemsep, nolistsep, leftmargin=0.8cm]
\item[1.] Predictive model trained under full automation $m_{\theta_0}$ without algorithmic triage, \ie, $\pi_{m_{\theta_0},b}(\xb) = \pi_0(\xb) = 0$ for all $\xb \in \Xcal$.

\item[2.] Predictive model trained under full automation $m_{\theta_0}$ with optimal algorithmic triage $\pi^*_{m_{\theta_0},b}$.

\item[3.] Predictive model trained under algorithmic triage $m_{\theta}$, with $b = 1$, with suboptimal algorithmic triage $\pi^*_{m_{\theta_0},b}$. 
Here, we use the triage policy that is optimal for the predictive model trained under full automation.
%

\item[4.] Predictive model trained under algorithmic triage ${m_{\theta}}$, with $b = 1$, with optimal algorithmic triage $\pi^*_{m_{\theta},b}$. 
%
\end{itemize}
In all the cases, we train the predictive models $m_{\theta_0}$ and $m_{\theta}$ using 
our method with $b = 0$ and $b = 1$, respectively.
Finally, we investigate the interplay between the accuracy of the above predictive models and the human experts and the structure of
the triage policies at a per-instance level. 

\xhdr{Results}
Figure~\ref{fig:first_synthetic} shows the training samples $(x, y)$ along with the predictions made by the predictive models $m_{\theta_0}$ and $m_{\theta}$ and 
the values of the triage policies $\pi_{0}(x)$, $\pi^{*}_{m_{\theta_0},b}$ and $\pi^{*}_{m_{\theta},b}$, as well as the losses achieved by the models and triage policies 
(1-4) on a per-instance level.
The results provide several interesting insights.

Since the predictive model trained under full automation $m_{\theta_0}$ seeks to generalized well 
across the entire feature space, the loss it achieves on a per-instance level is never too high, but 
neither too low, as shown in the left column of Panel (a). 
As a consequence, this model under no triage achieves the highest average loss among all alternatives, 
$L(\pi_0, m_{\theta_0}) = 0.0053$ (setting 1).
This may not come as a surprise since the mapping between feature and response variables does not lie 
within the hypothesis class of predictive mo\-dels used during training.
However, since the predictions by human experts are more accurate than those provided by the above 
model in some regions of the feature space, we can deploy the model with the optimal triage policy 
$\pi^{*}_{\theta_0,b}$ given by Theorem~\ref{theorem:optimal-policy} and lower the average loss to 
$L(\pi^{*}_{\theta_0,b}, m_{\theta_0}) = 0.0020$ (setting 2), as shown in the right column of Panel (a) and 
suggested by Proposition~\ref{prop:full-automation}.
\begin{figure*}[t]
     \centering
     \subfloat[Training losses by the predictive models $m_{\theta_t}$]{
	\scriptsize
	\stackunder[5pt]{\includegraphics[width=0.22\textwidth]{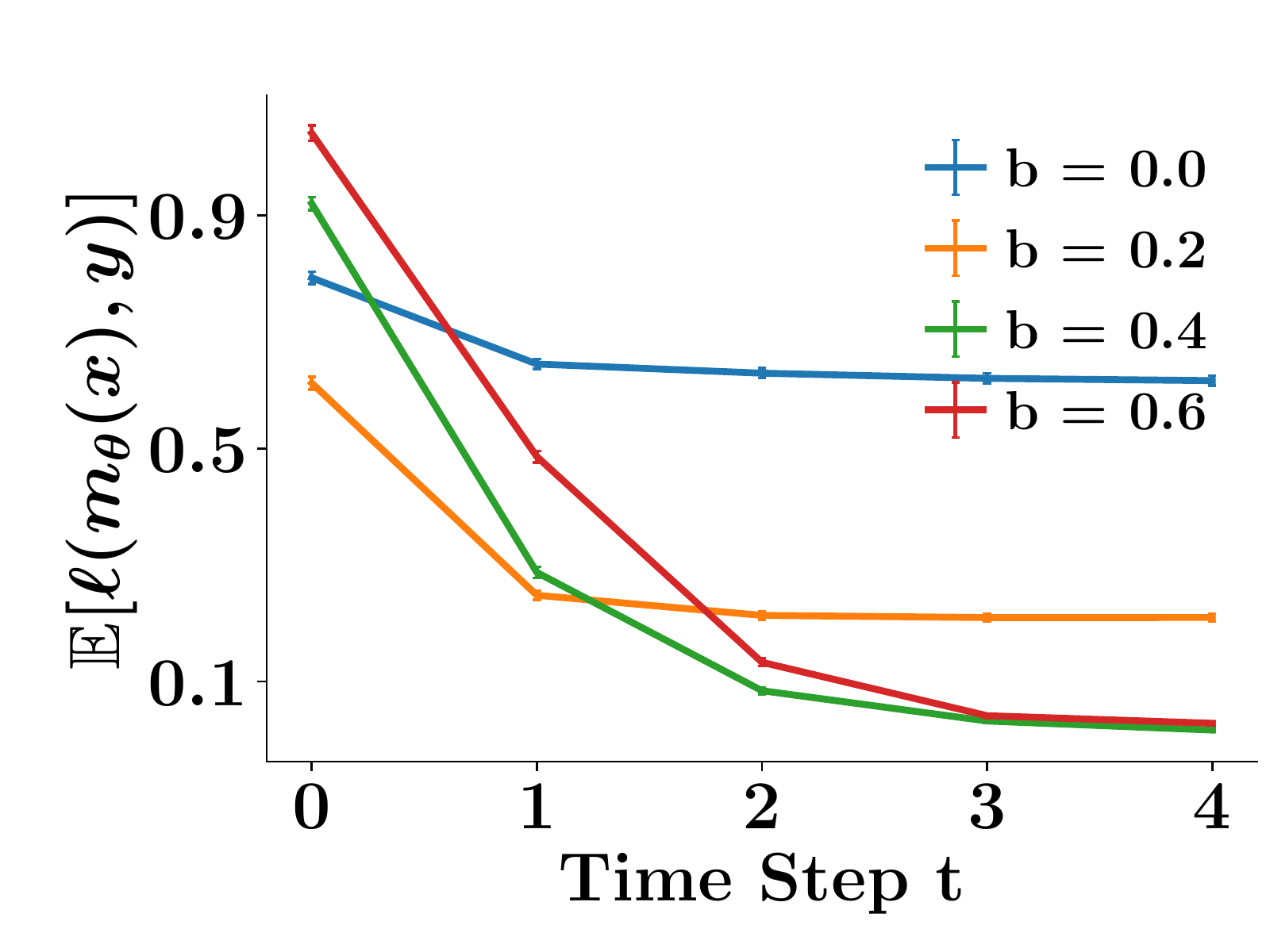}}{Hatespeech} 
    \stackunder[5pt]{\includegraphics[width=0.22\textwidth]{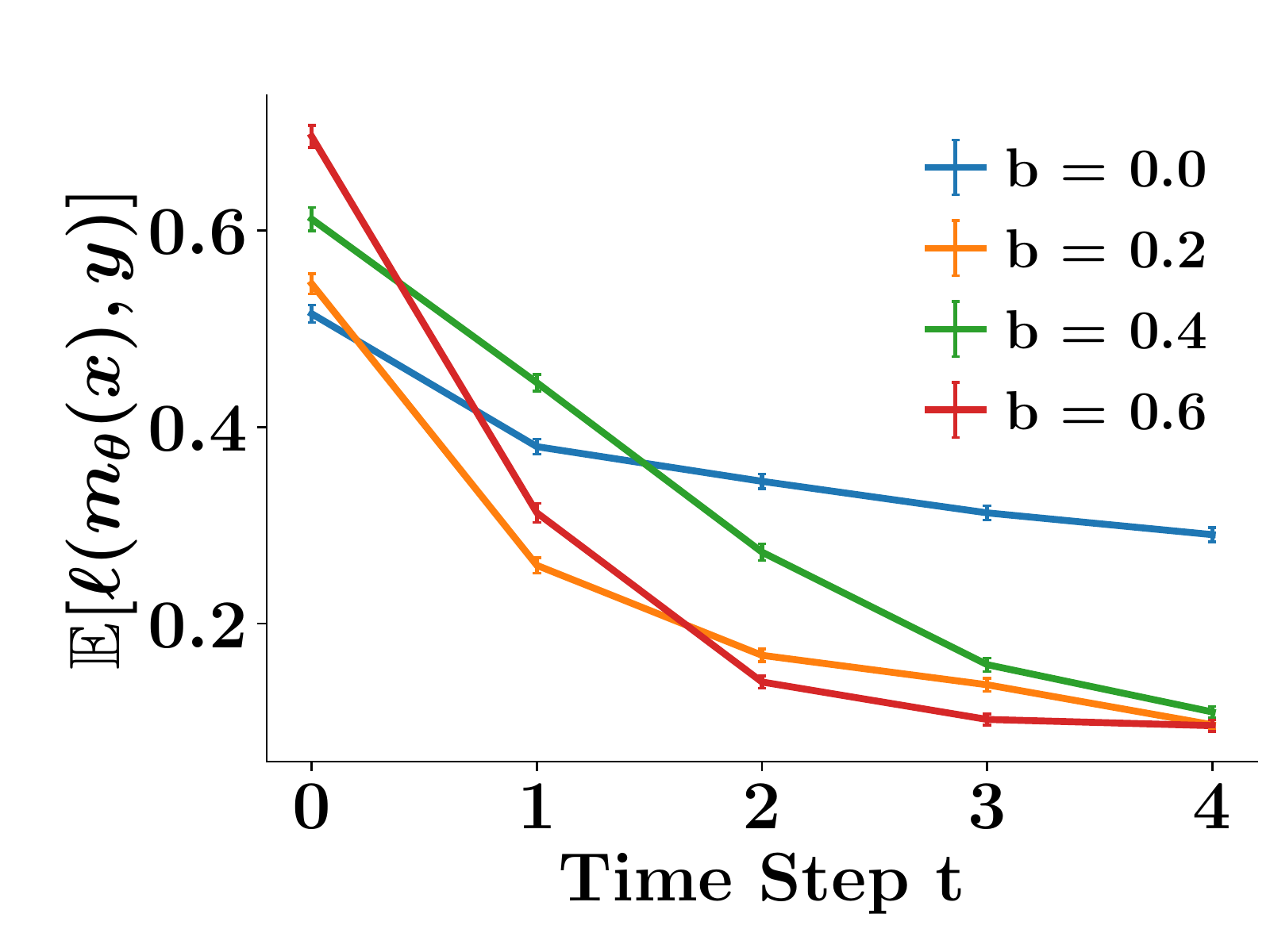}}{Galaxy zoo} 
%
     }  \hspace{5mm}
     \subfloat[Training losses by the triage policies $\hat{\pi}_{\gamma^{(i)}}$]{
        \scriptsize
	\stackunder[5pt]{\includegraphics[width=0.22\textwidth]{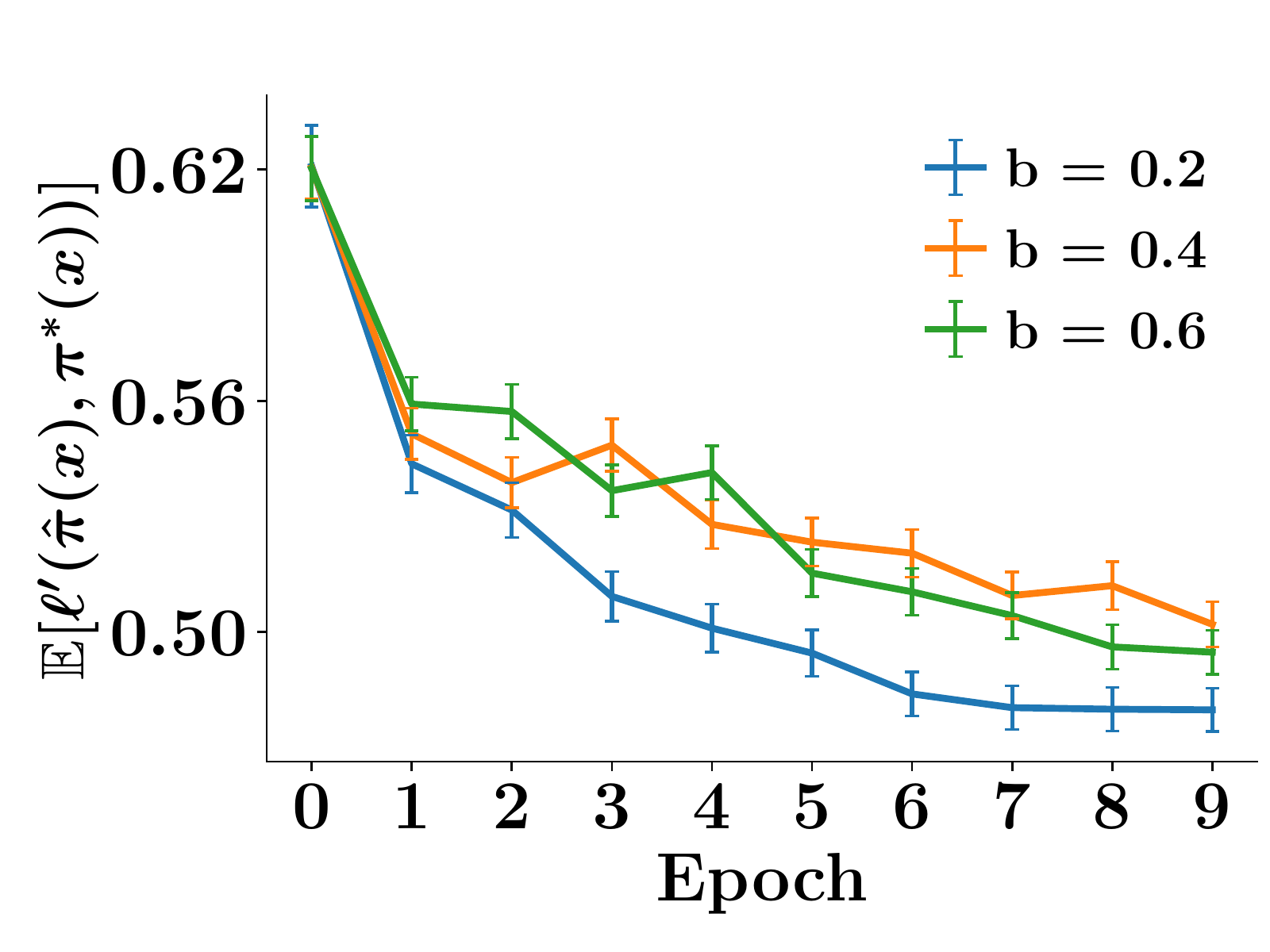}}{Hatespeech} 
	\stackunder[5pt]{\includegraphics[width=0.22\textwidth]{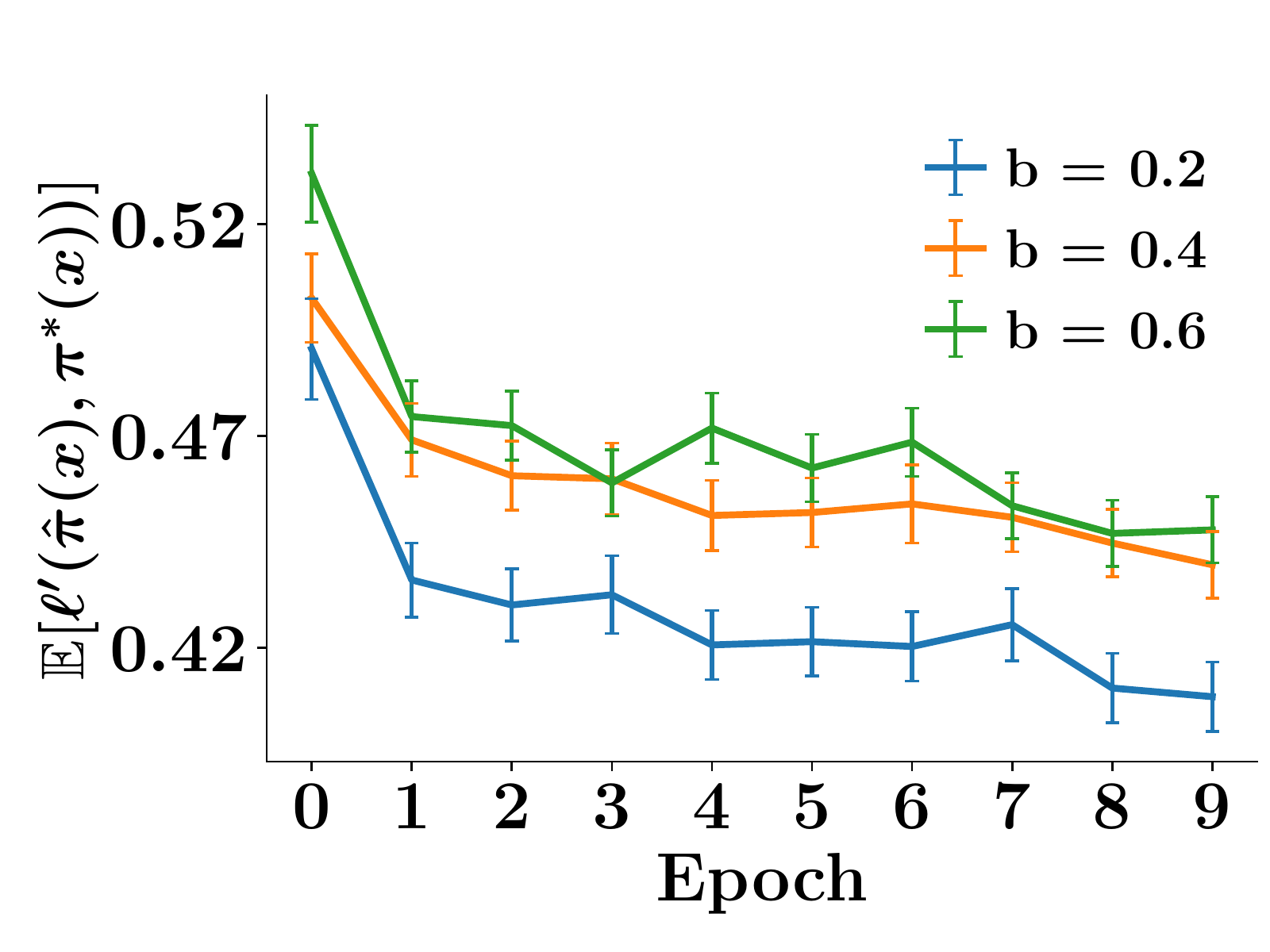}}{Galaxy zoo} 
%
     }
%
\caption{Average training losses achieved by the predictive models $m_{\theta_t}$ and the triage policies $\hat{\pi}_{\gamma^{(i)}}$ on the 
Hatespeech and Galaxy zoo datasets during training.
%
In Panel (a), each predictive model $m_{\theta_t}$ is the output of \textsc{TrainModel}$(\cdot)$ at step $t$ and
in Panel (b), each triage policy $\hat{\pi}_{\gamma^{(i)}}$ is the output of \textsc{TrainTriage}$(\cdot)$ at epoch $i$.
Both functions are defined in Algorithm~\ref{alg:sgd}. Error bars correspond to plus and minus one standard error.
%
} 
\label{fig:training-real}
\vspace{-4mm}
\end{figure*}

In contrast with the predictive model trained under full automation $m_{\theta_0}$, the predictive model trained 
under triage $m_{\theta}$ learns to predict very accurately the instances that lie in the regions of the feature 
space colored in green and yellow but it gives up on the regions colored in red and blue, where its predictions 
incur a very high loss, as shown in Panel (b).
However, these latter instances where the loss would have been the highest if the predictive model had to predict 
their response variables $y$ are those that the optimal triage policy $\pi^{*}_{m_{\theta},b}$ hand in to human experts 
to make predictions.
As a result, this predictive model under the optimal triage policy does achieve the lowest average loss among
all alternatives, $L(\pi^{*}_{\theta,b}, m_{\theta}) = 0.0009$ (setting 4), as suggested by Propositions~\ref{prop:suboptimality}
and~\ref{prop:suboptimality-2}.

Finally, our results also show that deploying the predictive model $m_{\theta}$ under a suboptimal triage policy 
may actually lead to a higher loss $L(\pi^{*}_{\theta_0,b}, m_{\theta}) = 0.0031$ (setting 3) than the loss achieved by the predictive 
model trained under full automation $m_{\theta_0}$ with its optimal triage policy $\pi^{*}_{\theta_0,b}$.
This happens because the predictive model $m_{\theta}$ is trained to work well \emph{only} on the instances $x$ 
such that $\pi^{*}_{m_\theta,b}(x) = 0$ and not necessarily on those with $\pi^{*}_{m_0,b}(x) = 0$.
%

\section{Experiments on Real Data} 
\label{sec:real}
In this section, we use our gradient-based algorithm in two binary and multi-class classification tasks in content moderation and scientific 
discovery,.
We first investigate the interplay between the accuracy of the predictive models and human experts and the structure of the optimal triage 
policies at different steps of the training process. 
Then, we compare the performance of our algorithm with several competitive baselines. 

\xhdr{Experimental setup} We use two publicly available datasets~\citep{hateoffensive,bamford2009galaxy}, one from an application in content moderation 
and the other for scientific discovery\footnote{\label{footnote:licenses}\scriptsize We chose these two particular applications because the corresponding datasets are among the only 
publicly available datasets that we found containing multiple human predictions per instance, necessary to estimate the human loss at an instance level, and 
a relatively large number of instances. 
The Hatespeech dataset is publicly available under MIT license and the Galaxy zoo dataset is publicly available under Creative Commons Attribution-Noncommercial-No 
Derivative Works 2.0 license.
Since our algorithm is general, it may be useful in other applications in which predictive models trained under full automation performs worse than humans in 
some instances (see Proposition~\ref{prop:full-automation}).}:

\noindent \emph{--- Hatespeech:} It consists of  $|\Dcal| = 24{,}783$ tweets containing lexicons used in 
hate speech. Each tweet is labeled by three to five human experts from Crowdflower as ``hate-speech'', ``offensive'', or ``neither''.
%

\noindent \emph{--- Galaxy zoo:} It consists of $|\Dcal| = 10{,}000$ galaxy images\footnote{\scriptsize The original Galaxy zoo dataset consists of $61{,}577$ images, 
however, we report results on a randomly chosen subset of $10{,}000$ images due to scalability reasons. We found similar results in other random subsets.}. 
Each image is labeled by $30$$+$ human experts as ``early type'' or ``spiral''. 
%
%

For each tweet in the Hatespeech dataset, we first generate a $100$ dimensional feature vector using fasttext~\citep{ft} as $\xb$, similarly as in~\cite{de2020aaai}. 
For each image in the Galaxy zoo dataset, we use its corresponding pixel map\footnote{\scriptsize The pixel maps for each image are available at \url{https://www.kaggle.com/c/galaxy-zoo-the-galaxy-challenge}} as $\xb$.
 %
Given an instance with feature value $\xb$, we estimate $P(h \given \xb)=\frac{n_\xb(h)}{\sum_{h'\in \Ycal} n_\xb(h')}$, where $n_{\xb}(h)$ denotes the 
number of human experts who predicted label $h$ 
%
and we set its true label to $y = \argmax_{h\in \Ycal} P(h\given \xb)$. 
Moreover, at test time, for each instance that the triage policy assigns to humans, we sample $h \sim P(h \given \xb)$.
\begin{figure*}[t]
\centering
\hspace{0.1cm}{\includegraphics[width=0.7\textwidth]{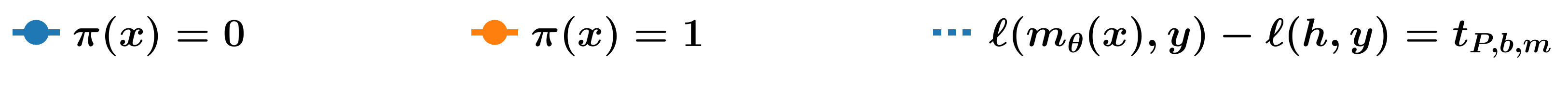}}\\[-3ex]
\subfloat{\includegraphics[width=0.28\textwidth]{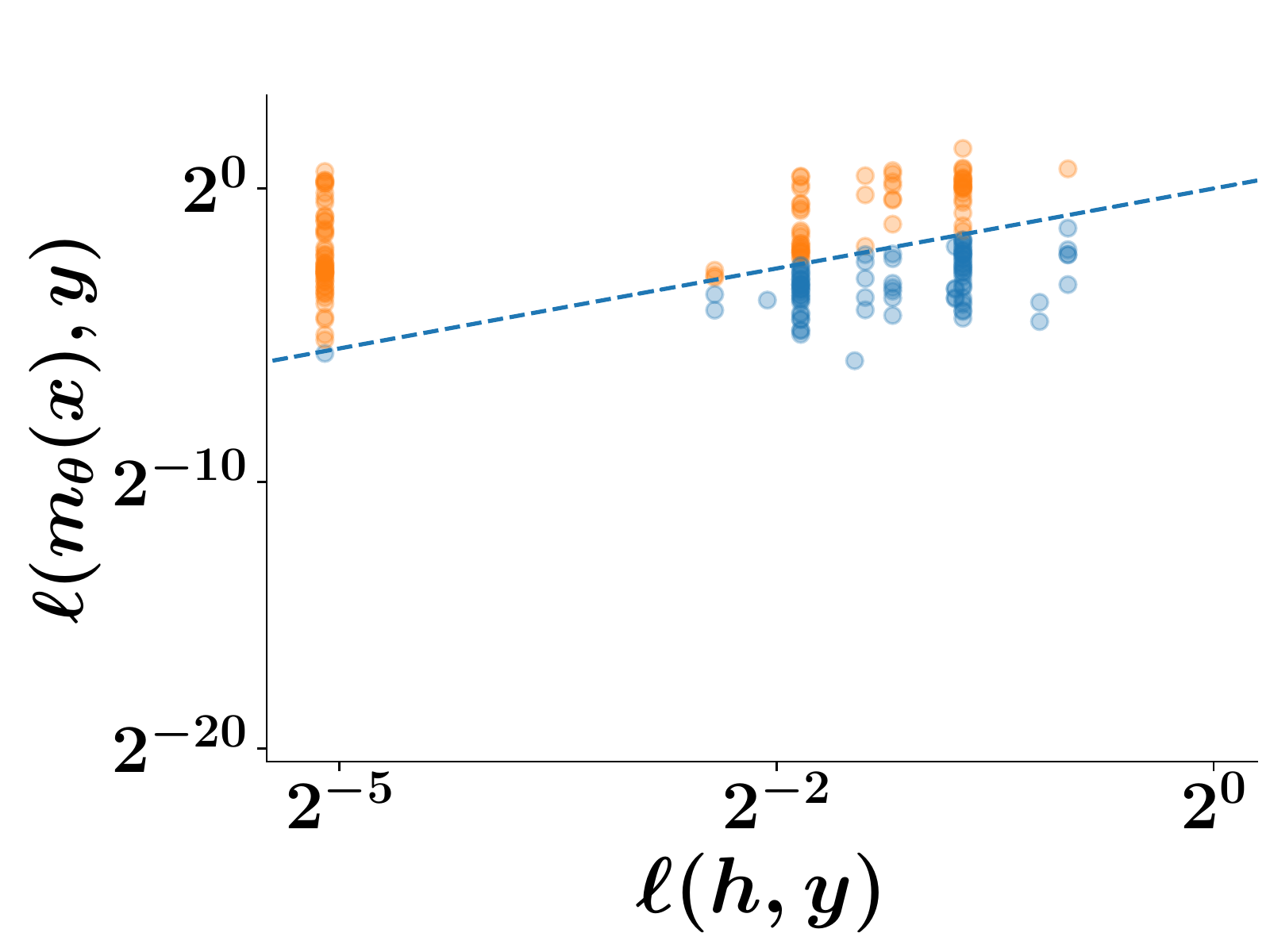}} \hspace{2mm}
\subfloat{\includegraphics[width=0.28\textwidth]{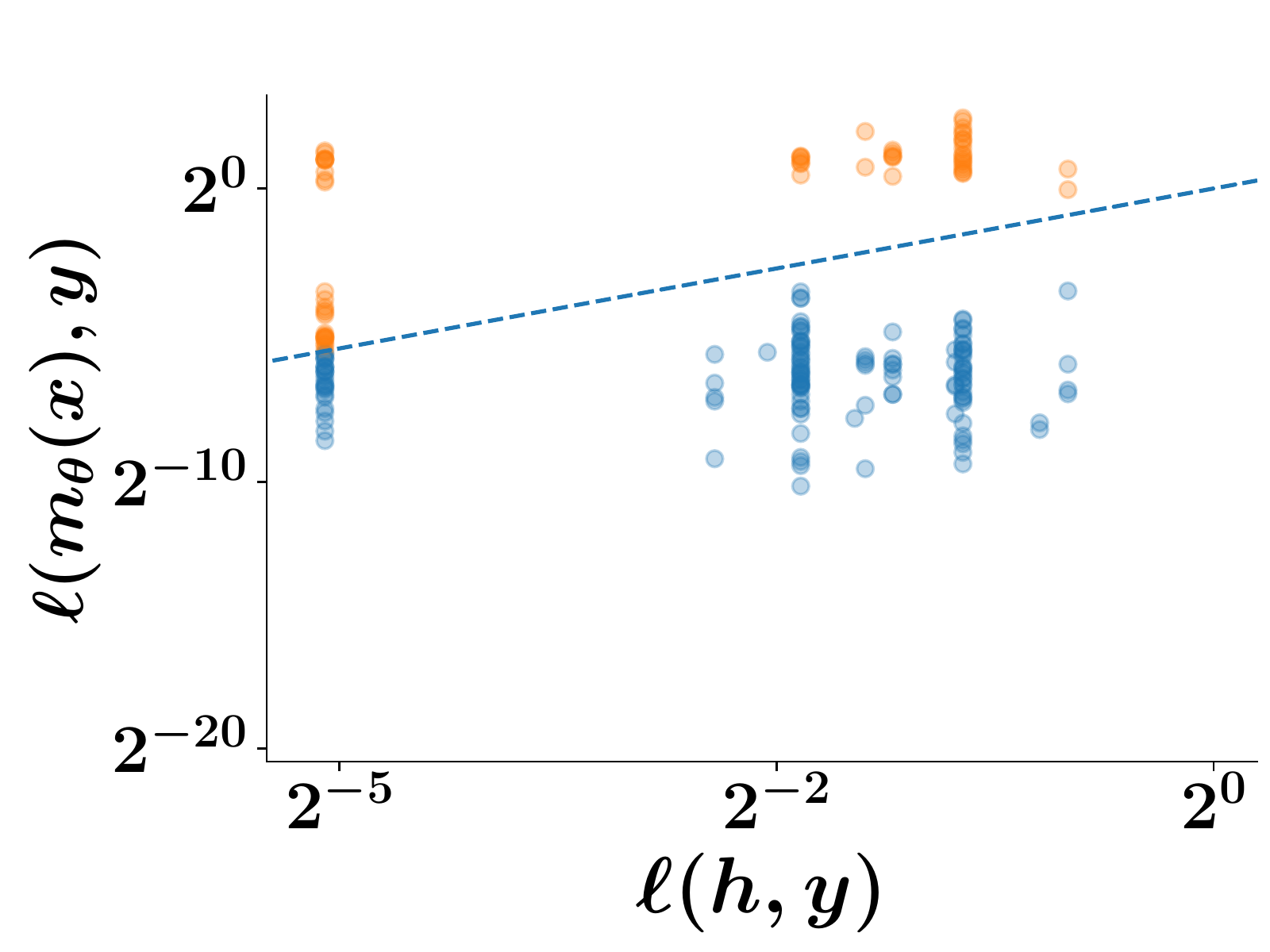}} \hspace{2mm}
 \subfloat{\includegraphics[width=0.28\textwidth]{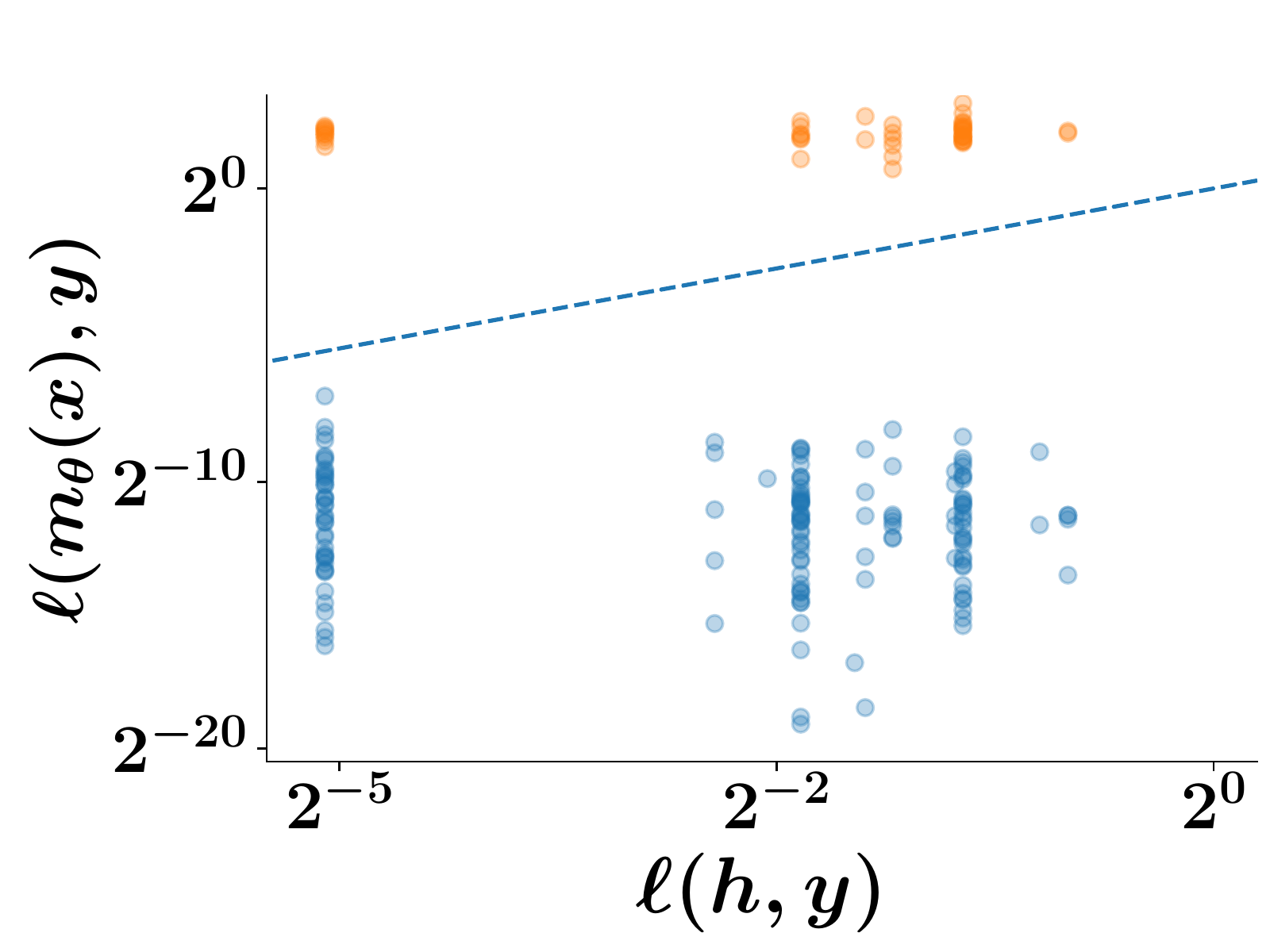}}\\
 \vspace{-4mm}
 \subfloat[Step $t=1$]{\setcounter{subfigure}{1}\includegraphics[width=0.28\textwidth]{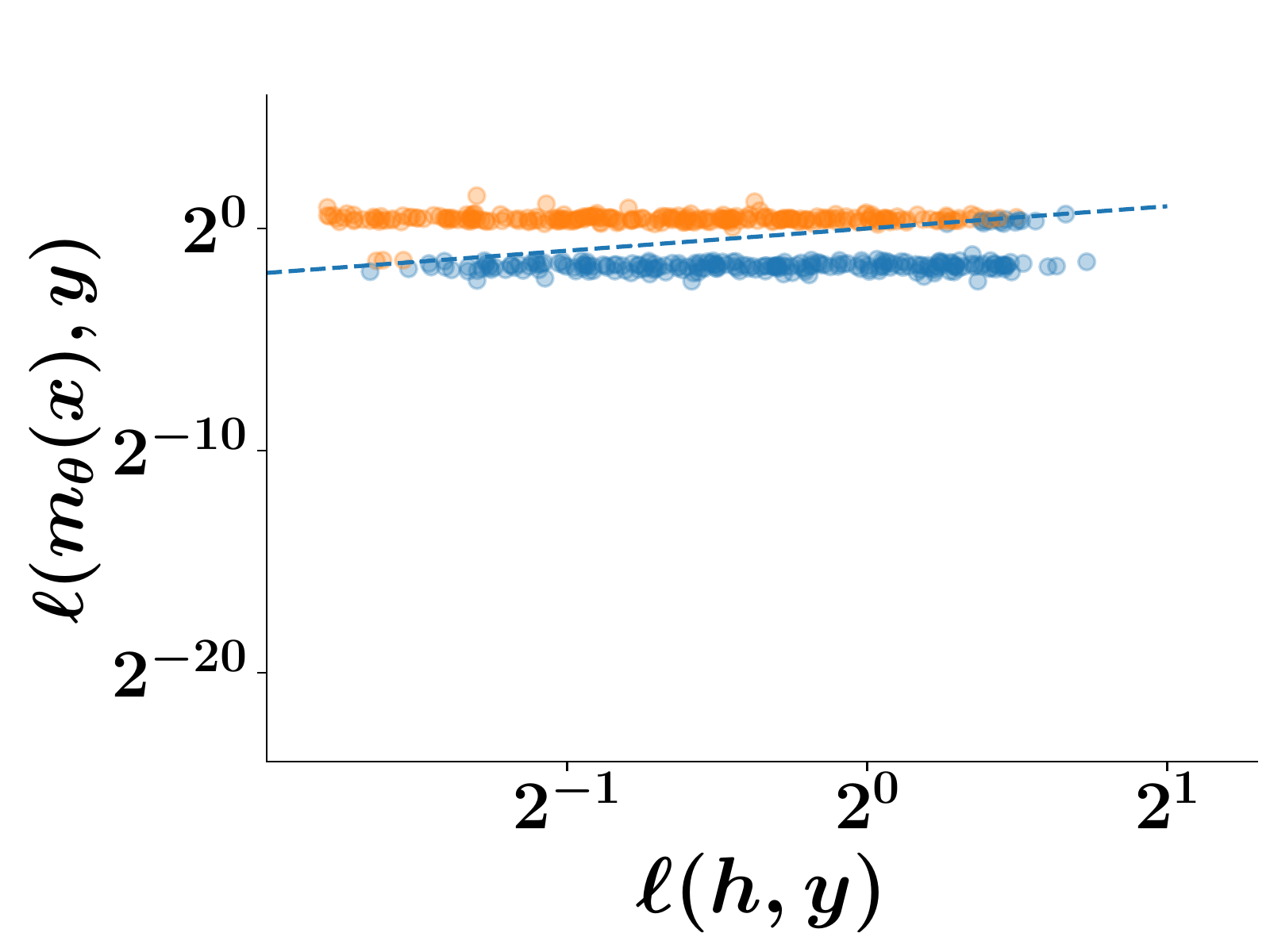}} \hspace{2mm}
\subfloat[Step $t=2$]{\includegraphics[width=0.28\textwidth]{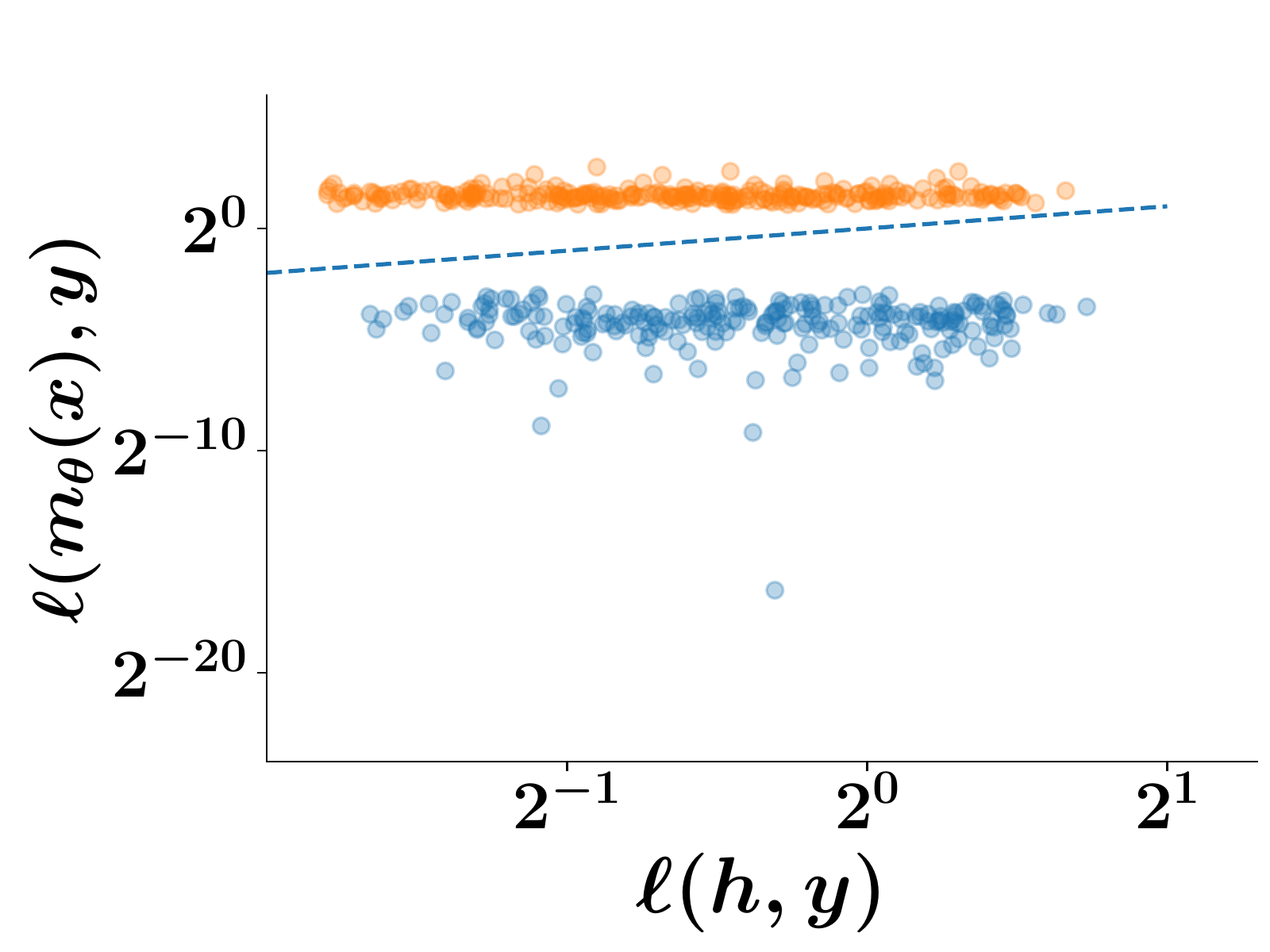}} \hspace{2mm}
 \subfloat[Step $t=3$]{\includegraphics[width=0.28\textwidth]{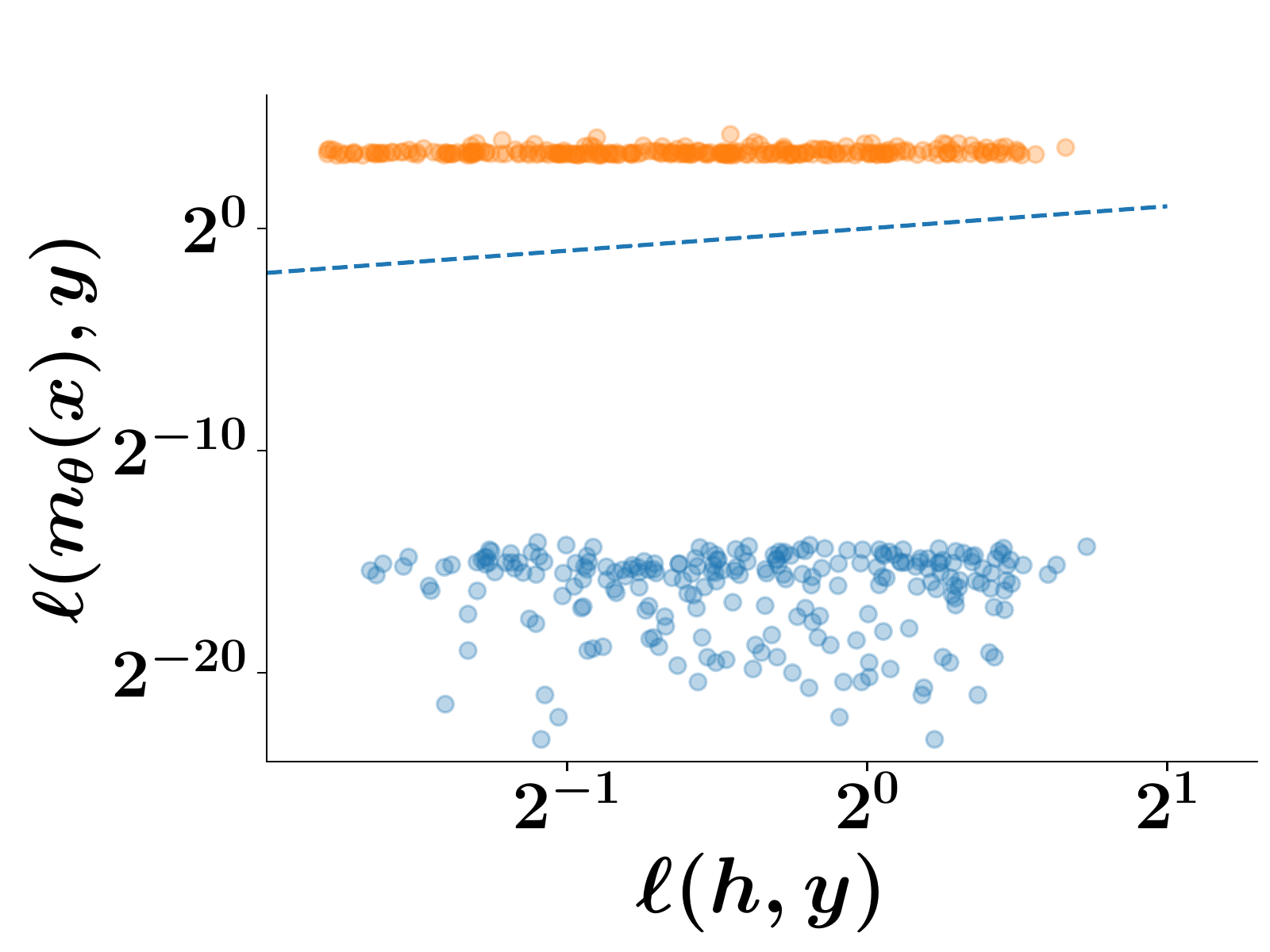}}
 \vspace{-2mm}
 \caption{Predictive model and expert losses at a per-instance level on a randomly selected subset of $500$ samples of the Hatespeech (top row) and Galaxy zoo (bottom row) datasets throughout the execution of 
our method
 during training. The maximum level of triage is set to $b=0.4$ for Hatespeech and $b=1.0$ for Galaxy zoo dataset. Each point corresponds to an individual
 instance and, for each instance, the color pattern indicates the triage policy value.}
\label{fig:scatter}
\vspace{-2mm}
\end{figure*}

In all our experiments, 
we consider the hypothesis class of probabilistic predictive models $\Mcal(\Theta)$ parameterized by softmax distributions, \ie,
\vspace{-2mm}
\begin{align}
 m_{\theta}(\xb) \sim  p_{\theta; \xb} =\text{Multinomial}\left(\left[\exp\left(\phi_{y, \theta}(\xb)\right)\right]_{y\in\Ycal}\right),\nn\\[-4ex]\nn
\end{align}
where, for the Hatespeech dataset, $\phi_{\bullet, \theta}$ is the convolutional neural network (CNN) by~\cite{kim2014convolutional}
and, for the Galaxy zoo dataset, it is the deep residual network by~\cite{he2015deep}.
During training, we use a cross entropy loss on the observed labels, \ie, $\ell(\hat{y},y)=- \log P(\hat{y}=y\given \xb)$. Here, if
an instance is assigned to the predictive model, we have that
\begin{equation*}
P(\hat{y}=y\given \xb) = \frac{\exp\left(\phi_{y, \theta}(\xb)\right)}{\sum_{y'} \exp\left(\phi_{y', \theta}(\xb)\right)},
\end{equation*} 
and, if an instance is assigned to a human expert, we have that $P(\hat{y}=y\given \xb) =   P(h = y \given \xb)$.
%
%
For the function $\hat{\pi}_{\gamma}(\xb)$, we use the class of logistic functions, \ie, $\hat{\pi}_{\gamma}(\xb)=\frac{1}{1+\exp(-\phi_{\gamma}(\xb))}$, where
%
$\phi_{\gamma}$ is the same CNN and deep residual network as in the predictive model, respectively. 
%
Here, we also use the cross entropy loss, 
\ie, $\ell'(\hat{\pi}_{\gamma}(\xb), \pi^*_{m_{\theta_T},b}(\xb))= -\pi^*_{m_{\theta_T},b}(\xb) \log \hat{\pi}_{\gamma}(\xb) -(1-\pi^*_{m_{\theta_T},b}(\xb)) \log (1-\hat{\pi}_{\gamma}(\xb)) $.
%
In each experiment, we used 60\% samples for training, 20\% for validation and 20\% for testing. 
Refer to Appendix~\ref{app:exp-details} for additional details on the experimental setup.
%

\xhdr{Results} 
First, we look at the average loss achieved by the predictive models $m_{\theta_t}$ and triage policies $\hat{\pi}_{\gamma^{(i)}}$ throughout the execution of 
our method
during training.
%
%
Figure~\ref{fig:training-real} summarizes the results, which reveal several insights. 
For small values of the triage level, the models $m_{\theta_t}$ aim to generalize well across a large portion of the feature space. As a result, 
they incur a large training loss, as shown in Panel (a). In contrast, for $b\ge 0.4$, the models $m_{\theta_t}$ are trained to generalize across a smaller region 
of the feature space, which leads to a considerably smaller training loss. 
However, for such a high triage level, the overall performance of our method is also contingent on how well $\hat{\pi}_{\gamma}$ approximates the optimal triage policy. 
%
Fortunately, Panel (b) shows that, as epochs increase, the average training loss of $\hat{\pi}_{\gamma^{(j)}}$ decreases. 
%
%
Appendix~\ref{app:triage} validates further the trained approximate triage policies $\hat{\pi}_{\gamma}$.

Next, we compare the predictive model and the human expert losses per training instance throughout the execution of the 
our method during training.
Figure~\ref{fig:scatter} summarizes the results. At each step $t$, we find that the optimal triage policies $\pi^*_{m_{\theta_t},b}$ hands in to human experts those instances 
(in orange) where the loss would have been the highest if the predictive model had to predict their response variables $y$.
%
%
Moreover, at the beginning of the training process (\ie, low step values $t$), since the predictive model $m_{\theta_t}$ seeks to generalize across a large portion of the feature 
space, the model loss remains similar across the feature space. 
However, later into the training process (\ie, high step values $t$), the predictive models $m_{\theta_t}$ focuses on predicting more accurately the samples that the triage policy hands in
to the model, achieving a lower loss on those samples, and gives up on the remaining samples, where it achieves a high loss.

Finally, we compare the performance of our method against four baselines in terms of test misclassification error $P(\hat{y}\neq y)$. Refer to Appendix~\ref{app:exp-details} for 
more details on the baselines, which we refer to as confidence-based triage~\citep{bansal2020optimizing}, score-based triage~\citep{raghu2019algorithmic}, surrogate-based triage~\citep{sontag2020} and full automation triage\footnote{\scriptsize Among all the baselines we are aware of~\citep{raghu2019algorithmic, de2020aaai,de2021aaai, wilder2020learning, bansal2020optimizing, sontag2020}, we did not compare with~\cite{de2020aaai,de2021aaai} because they only allow linear models and SVMs and we did not compare with~\cite{wilder2020learning} because they did not provide enough 
details to implement their method.}. 
%
Figure~\ref{fig:baseline} summarizes the results for different values of the maximum triage level $b$. 
We find that, in each dataset, our algorithm is able to find the predictive model and the triage policy with the lowest misclassification
across the entire span of $b$ values.
In the Hatespeech dataset, they are those corresponding to $b = 0.8$ and, in the Galaxy zoo dateset, they are those corresponding to $b = 0.6$.
One could view these values of maximum triage levels as the optimal automation levels (within the level values we experimented with).
%
%
\begin{figure*}[t]
\centering
\hspace{0.1cm}{\includegraphics[width=0.6\textwidth]{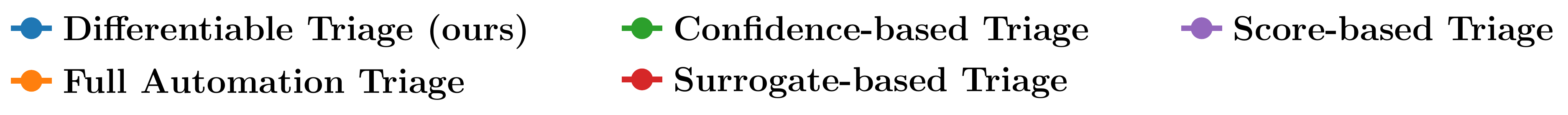}}\\[-3ex]
\subfloat[Hatespeech]{\includegraphics[width=0.28\textwidth]{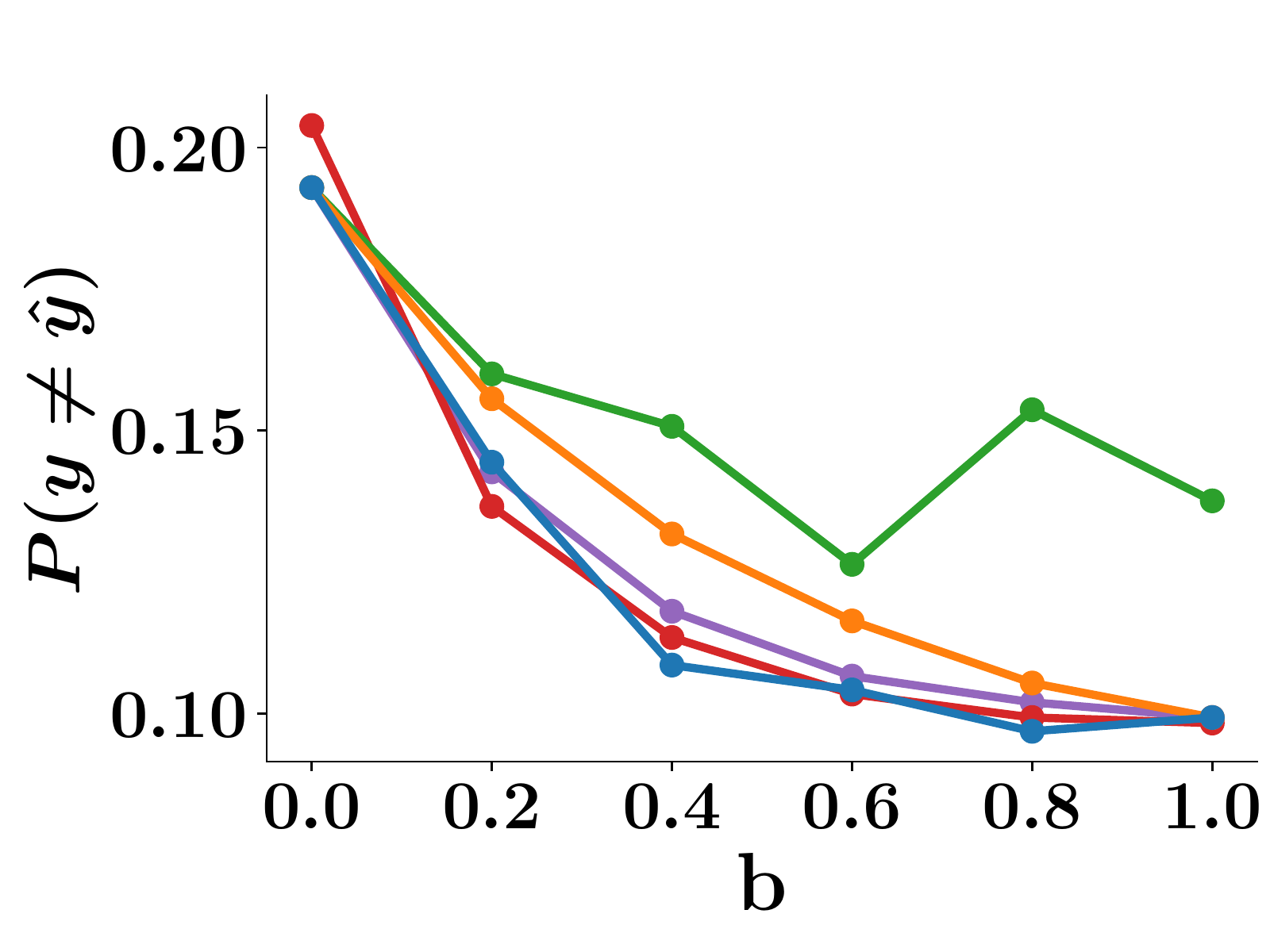}} \hspace{6mm}
\subfloat[Galaxy zoo]{\includegraphics[width=0.28\textwidth]{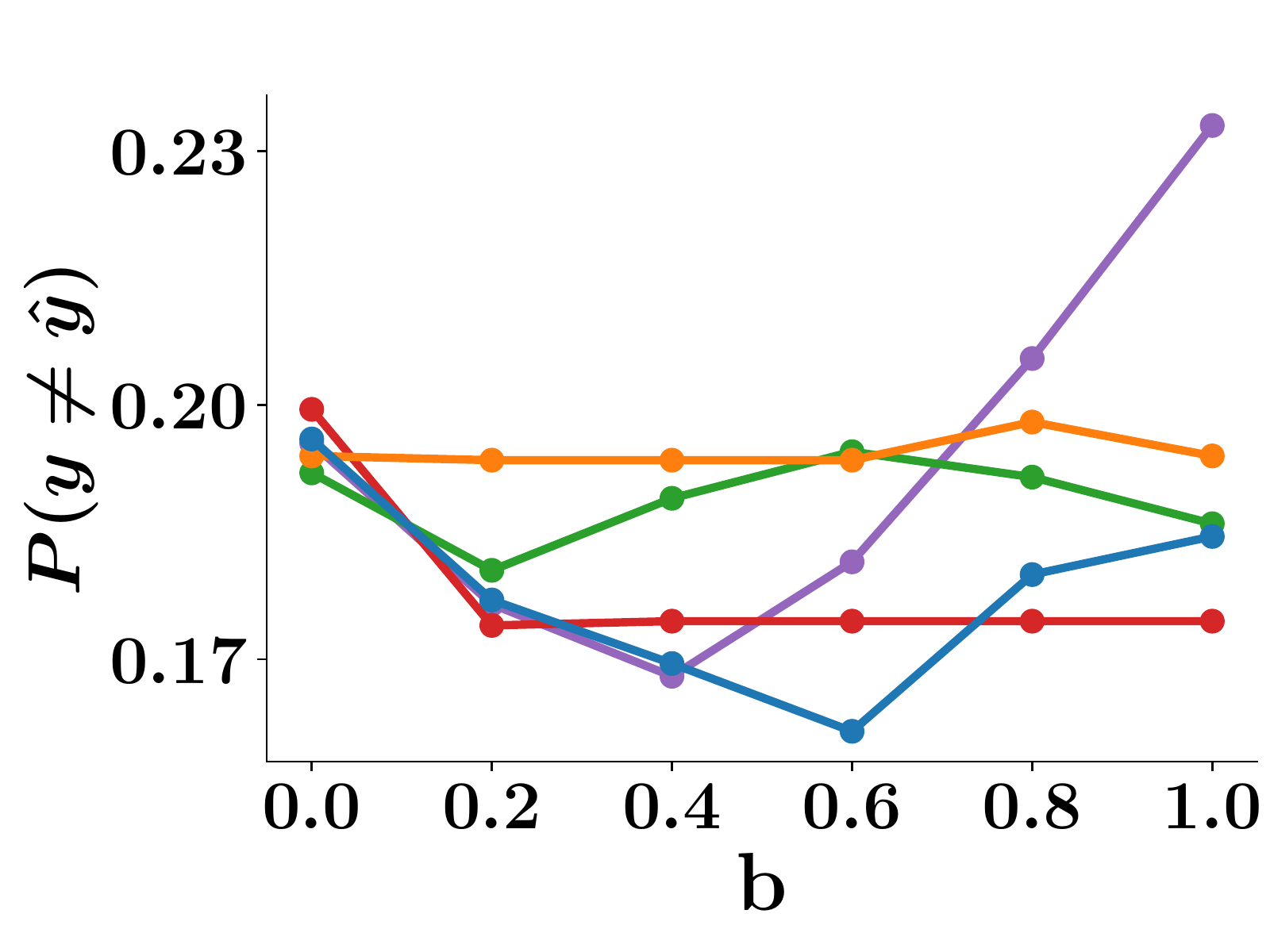}} 
\caption{Misclassification test error $P(\hat{y}\neq y)$ against the triage level $b$ on the Hatespeech and Galaxy zoo datasets for our algorithm, 
confidence-based triage~\citep{bansal2020optimizing}, score-based triage~\citep{raghu2019algorithmic}, surrogate-based triage~\citep{sontag2020} 
and full automation triage. Appendix~\ref{app:exp-details} contains more details on the baselines.
}
\label{fig:baseline}
\vspace{-5mm}
\end{figure*}


\section{Conclusions}
\label{sec:conclusions}
In this paper, we have contributed towards a better understanding of supervised learning under algo\-rith\-mic triage. 
We have first identified under which circumstances predictive models may benefit from algorithmic triage, including
those trained for full automation.
Then, given a predictive model and desired level of triage, we have shown that the optimal triage policy is a deterministic 
threshold rule in which triage decisions are derived deterministically from the model and human per-instance errors.
Finally, we have introduced a practical algorithm to train supervised learning models under triage and have shown that it 
outperforms several competitive baselines.

Our work also opens many interesting venues for future work. For example, we have assumed that each instance
is predicted by either a predictive model or a human expert. 
However, there may be many situations in which human experts predict all instances but their predictions are informed by 
a predictive model~\citep{lubars2020ask}. 
We have shown that our algorithm is guaranteed to converge to a local minimum of the empirical risk, however, it would be 
interesting to analyze the convergence rate and the generalization error.
%
%
Finally, it would be valuable to assess the performance of supervised learning models under algorithmic triage using 
interventional experiments on a real-world application.
%


\xhdr{Acknowledgment}
Okati and Gomez-Rodriguez acknowledge support from the European Research Council (ERC) under the European Union'{}s Horizon 2020 research and innovation 
programme (grant agreement No. 945719). De has been partially supported by a DST Inspire Faculty Award.

\bibliography{refs}

\clearpage
\newpage
\onecolumn

\appendix

\section{Proofs} \label{app:proofs}

\xhdr{Proof of Proposition~\ref{prop:unbiased-biased}}
Due to Jensen'{}s inequality and the fact that, by assumption, the distribution of human predictions $P(h \given \xb)$ is not a point-mass, it holds 
that $\EE_{h}[\ell(h(\xb),y) \given \xb\,] > \ell(\mu_h(\xb),y)$. Hence, 
\begin{align}
  \EE_{\xb, y, h}\left[ (1-\pi(\xb)) \, \ell(\mac(\xb), y) + \pi(\xb) \,   \ell(h(\xb), y) \right] > \EE_{\xb, y}\left[ (1-\pi(\xb)) \, \ell(\mac(\xb), y) + \pi(\xb) \,   \ell(\mu_{h}(\xb), y) \right].
\end{align}
%

\xhdr{Proof of Proposition~\ref{prop:full-automation}}
Let $\pi(\xb) = \II(\xb \in \Vcal)$. Then, we have:
\begin{align*}
    L(\pi, \mac_0^{*}) 
    & = \int_{\xb \in \Xcal\setminus \Vcal} \EE_{y | \xb} \left[ \ell(\mac_0^{*}(\xb), y) \right] \, dP + \int_{\xb \in \Vcal} \EE_{y, \hsample  | \xb} \left[ \ell(\hsample, y) \right] \, dP \\
    & \overset{(i)} < \int_{\xb \in \Xcal\setminus \Vcal} \EE_{y | \xb} \left[ \ell(\mac_0^{*}(\xb), y) \right] \, dP + \int_{\xb \in \Vcal} \EE_{y | \xb} \left[ \ell(\mac_0^{*}(\xb), y) \right] \, dP \\
    & = \int_{\xb \in \Xcal} \EE_{y | \xb} \left[ \ell(\mac_0^{*}(\xb), y) \right] \, dP \\
    & \overset{(ii)}= L(\pi_0, \mac_0^{*}),
\end{align*} 
where inequality $(i)$ holds by assumption and equality $(ii)$ holds by the definition of $\pi_0(\xb)$.

\xhdr{Proof of Theorem~\ref{theorem:optimal-policy}}
We first provide the proof of the unconstrained case. First, 
we note that, 
 \begin{align*}
     L(\pi, \mac) &=\EE_{\xb, \hsample } \left[ (1-\pi(\xb)) \, \EE_{y | \xb}[\ell(\mac(\xb), y) ] + \pi(\xb) \, \EE_{y | \xb}[ \ell(\hsample, y)] \right] \\
     & = \EE_{\xb} \left[ (1- \pi(\xb)) \, \EE_{y | \xb}[\ell(\mac(\xb), y)] + \pi(\xb) \, \EE_{y, \hsample | \xb}[ \ell(\hsample, y)] \right] \\
    & = \EE_{\xb} \left[\pi(\xb) \bigg[    \EE_{y, \hsample | \xb}[  \ell(\hsample, y) ] 
    - \EE_{y | \xb}[\ell(\mac(\xb), y)] \bigg] \right] + \EE_{\xb, y}[\ell(\mac(\xb), y) ]
 \end{align*}
Since the second term in the above equation does not depend on $\pi$, we can find the optimal policy $\pi$ by solving the following optimization problem:
\begin{align*}
    \underset{\pi}{\text{minimize}} &  \quad  \EE_{\xb} \left[\pi(\xb) \bigg[    \EE_{y, \hsample  | \xb}[ \ell(\hsample, y) ] 
    - \EE_{y | \xb}[\ell(\mac(\xb), y)]
    \bigg] \right]\\
    \text{subject to} &  \quad 0 \le  \pi(\xb) \le 1 \quad \forall \ \xb\in \Xcal. 
\end{align*} 
Note that the above problem is a linear program and it decouples with respect to $\xb$. Therefore, for each $\xb$,  the optimal solution is clearly given by:
\begin{equation*} 
\pi_{\mac}^{*}(d = 1 \,|\, \xb) =
\begin{cases}
1 & \text{if} \,\, \EE_{y | \xb}[\ell(\mac(\xb), y) -    \EE_{\hsample | \xb}[\ell(\hsample, y)]]  > 0 \\
0 & \text{otherwise}
\end{cases}
\end{equation*}
Next, we provide the proof of the constrained case.
Here, we need to solve the following optimization problem:
\begin{align*}
 \underset{\pi}{\text{minimize}}& \quad \EE_{\xb} \left[\pi(\xb) \bigg[    \EE_{y, \hsample  | \xb}[\ell(\hsample, y) ] 
    - \EE_{y | \xb}[\ell(\mac(\xb), y)]
    \bigg] \right] \\
\text{subject to} & \quad \EE_{\xb} [\pi(\xb)] \leq b, \\
& \quad 0 \le  \pi(\xb) \le 1 \quad \forall \ \xb\in \Xcal. 
\end{align*}
To this aim, we consider the dual formulation of the optimization problem, where we only introduce a Lagrangian multiplier $\tau_{P,b}$ for the first constraint, \ie,
\begin{align}
  \underset{\tau_{P,b} \ge 0}{\text{maximize}}\, \underset{\pi}{\text{minimize}}&\ \     \EE_{\xb} \left[\pi(\xb) \bigg[     \EE_{y,  \hsample | \xb}[ \ell(\hsample, y) ] 
    - \EE_{y | \xb}[\ell(\mac(\xb), y)]
    \bigg] \right] \nonumber \\
     &\qquad\qquad \qquad 
    + \EE_{\xb} \left[\tau_{P,b}\left(\pi(\xb)- b \right)\right] \\
\text{subject to}
& \quad 0 \le  \pi(\xb) \le 1 \quad \forall \ \xb\in \Xcal.
\end{align}
The inner minimization problem can be solved using the similar argument for the unconstrained case. Therefore, we have:
\begin{equation*} 
\pi_{{\mac}^{*},b}(d = 1 \,|\, \xb) =
\begin{cases}
1 & \text{if} \,\, \EE_{y | \xb}[\ell(\mac(\xb), y) -   \EE_{\hsample | \xb}[ \ell(\hsample, y)]]  > t_{P, b, m} \\
0 & \text{otherwise}
\end{cases}
\end{equation*}
where
\begin{align*}
      t_{P,b}= \underset{\tau_{P,b} \ge 0}{\text{argmax}}\,\, \EE_{\xb} \left[ \min \left( \EE_{y | \xb}[\EE_{\hsample | \xb} [\ell(\hsample, y)] - \ell(\mac(\xb), y)]] + \tau_{P,b}, \, 0\right)  - \tau_{P,b}\,b \right]
\end{align*}

 \xhdr{Proof of Proposition~\ref{prop:suboptimality}}
The optimal predictive model $m_{\theta^{*}_0}$ under full automation within a parameterized hypothesis class of predictive models $\Mcal(\Theta)$ satisfies that
\begin{equation} \label{eq:zero-gradient}
\nabla_{\theta} \left. L(\pi_0, \mac_{\theta})\right|_{\theta = \theta^{*}_0} = \EE_{\xb, y} \left[ \nabla_{\theta} \left. \ell(m_{\theta}(\xb), y)\right|_{\theta = \theta^{*}_0} \right]  = \bm{0}
\end{equation}
and the optimal predictive model $m_{\theta^{*}}$ under $\pi^{*}_{m_{\theta^{*}},b}$ satisfies that
\begin{equation}
    \nabla_{\theta} \left. L(\pi^{*}_{\mac_{\theta},b}, \mac_{\theta})\right|_{\theta = \theta^{*}} = \bm{0}.
\end{equation}
%
%
Now we have that
\begin{align}
\nabla_{\theta} \left. L(\pi^{*}_{\mac_{\theta},b}, \mac_{\theta})\right|_{\theta = \theta^{*}_0} &= \nabla_{\theta} \left. L(\pi_0, \mac_{\theta})\right|_{\theta = \theta^{*}_0} \nn\\
&\qquad - \nabla_{\theta} \left. \EE_{\xb} \left[ \textsc{Thres}_{t_{P, b, m}}\left( \EE_{y | \xb} \left[ \ell(\mac(\xb), y) - \EE_{\hsample | \xb}[\ell(\hsample, y)] \right], 0 \right) \right]\right|_{\theta = \theta^{*}_0}\nonumber \\
&= 0 - \int_{\xb \in  \Vcal} \EE_{y | \xb}\left[ \left. \nabla_{\theta} \ell(\mac_{\theta}(\xb), y) \right|_{\theta = \theta^{*}_0} \right] \, dP - \int_{\xb \in \Xcal\backslash\Vcal} 0 \, dP
  \neq \bm{0}.\label{eq:inter-grad}
\end{align}
where we have used that
\begin{equation*}
    \nabla_x \textsc{Thres}_{t_{P, b, m}}(f(x), 0) = 
    \begin{cases}
    \nabla_{x} f(x) & \text{if} \,\, f(x) > t_{P, b, m} \\
    0 & \text{if} \,\, f(x) < t_{P, b, m}.
    \end{cases}
\end{equation*}
Hence, we can immediately conclude that $L(\pi^{*}_{m_{\theta^{*}_0},b}, m_{\theta^{*}_0}) > \min_{\theta \in \Theta} L(\pi^{*}_{\mac_{\theta},b}, \mac_{\theta})$.

\xhdr{Proof of Proposition~\ref{prop:suboptimality-2}}
Under triage policy $\pi^{*}_{m_{\theta'},b}$, we have that:
\begin{align*}
\nabla_{\theta} \left. L(\pi^{*}_{m_{\theta'},b}, \mac_{\theta}) \right|_{\theta = \theta'} &= \nabla_{\theta} \left. \EE_{\xb} \big[ (1-\pi^{*}_{m_{\theta'}}(\xb)) \, \EE_{y | \xb}[\ell(\mac_{\theta}(\xb), y) ] + \pi^{*}_{m_{\theta'},b}(\xb) \, \EE_{y, \hsample | \xb} \left[ \ell(h, y) \right] \big] \right|_{\theta = \theta'} \\
&= \EE_{\xb} \big[ (1-\pi^{*}_{m_{\theta'},b}(\xb)) \, \EE_{y | \xb}[ \nabla_{\theta} \left. \ell(\mac_{\theta}(\xb), y) \right|_{\theta = \theta'} ] \big] \\
&=  \int_{\xb \in \Vcal} \EE_{y | \xb}\left[ \left. \nabla_{\theta} \ell(\mac_{\theta}(\xb), y) \right|_{\theta = \theta'} \right] \neq \bm{0},
\end{align*}
where $\Vcal = \{ \xb \given \pi_{m_{\theta'},b}(\xb) = 0 \}$. Hence, we can immediately conclude that $L(\pi^{*}_{m_{\theta}b}, m_{\theta'}) > \min_{\theta \in \Theta} L(\pi^{*}_{m_{\theta},b}, \mac_{\theta})$.

\xhdr{Proof of Proposition~\ref{prop:GD-error-change}}
Since $\pi^{*}_{m_{\theta_{t}},b} =\argmin_{\pi} L(\pi, m_{\theta_{t}})$, we have that:
 \begin{align}
  L(\pi^{*}_{m_{\theta_{t}},b}, m_{\theta_{t}} ) \le L(\pi^{*}_{m_{\theta_{t-1}},b}, m_{\theta_{t}} ) \label{eq:intermed1}
 \end{align}
Then, if $\theta_t ^{(i)}$ is computed from $\theta_t ^{(i-1)}$ using Eq.~\ref{eq:sgd-step-1}, then we have that~\citep[Eq. 9.17]{boyd2004convex}:
\begin{align}
 L(\pi^{*}_{m_{\theta _{t-1}},b}, m_{\theta^{(i)} _{t}} )& \le L(\pi^{*}_{m_{\theta _{t-1}},b}, m_{\theta^{(i-1)} _{t}} )\nn\\
 &\quad + \nabla_{\theta} L(\pi^{*}_{m_{\theta _{t-1}},b}, m_{\theta^{(i-1)} _{t}} )^\top  (\theta^{(i)} _{t} - \theta^{(i-1)} _{t}) +  \frac{\Lambda}{2} \bnm{\theta^{(i-1)} _{t} - \theta^{(i)} _{t}}^2\nn\\
 & \overset{(a)}{=} L(\pi^{*}_{m_{\theta _{t-1}},b}, m_{\theta^{(i-1)} _{t}} )  -\alpha^{(i-1)}\nabla_{\theta} L(\pi^{*}_{m_{\theta _{t-1}},b}, m_{\theta^{(i-1)} _{t}} )^\top \nabla_{\theta} L(\pi^{*}_{m_{\theta _{t-1}},b}, m_{\theta^{(i-1)} _{t}} )\nn\\
 &\qquad\qquad+(\alpha^{(i-1)} )^2 \frac{\Lambda}{2} \bnm{\nabla_{\theta} L(\pi^{*}_{m_{\theta _{t-1}}}, m_{\theta^{(i-1)} _{t}} )}^2 \nn\\
 & {=} L(\pi^{*}_{m_{\theta _{t-1}},b}, m_{\theta^{(i-1)} _{t}} )-\left(\alpha^{(i-1)} -(\alpha^{(i-1)} )^2 \frac{\Lambda}{2}\right) \bnm{\nabla_{\theta} L(\pi^{*}_{m_{\theta _{t-1}},b}, m_{\theta^{(i-1)} _{t}} )}^2 \nn\\
&  \overset{(b)}{<} L(\pi^{*}_{m_{\theta _{t-1}},b}, m_{\theta^{(i-1)} _{t}} )-  \frac{\alpha^{(i-1)}}{2}  \bnm{\nabla_{\theta} L(\pi^{*}_{m_{\theta _{t-1}},b}, m_{\theta^{(i-1)} _{t}} )}^2 \nn \\
&  {<}\,   L(\pi^{*}_{m_{\theta _{t-1}},b}, m_{\theta^{(i-1)} _{t}} ), \label{eq:theta-dec}
\end{align}
where equality $(a)$ follows from the fact that
\begin{align}
 \theta_{t}^{(i)} - \theta_{t}^{(i-1)}= - \alpha^{(i-1)} \nabla_{\theta} \left. L(\pi^{*}_{m_{\theta_{t-1}},b}, m_{\theta}) \right|_{\theta=\theta_{t}^{(i-1)}}
\end{align}
and inequality $(b)$ follows by assumption, \ie, $\alpha^{(i-1)}\Lambda < 1$.

Eq.~\ref{eq:theta-dec} directly implies that 
\begin{equation*}
L(\pi^{*}_{m_{\theta _{t-1}},b}, m_{\theta_{t}} ) <  L(\pi^{*}_{m_{\theta _{t-1}},b}, m_{\theta^{(0)} _{t}} ) = L(\pi^{*}_{m_{\theta _{t-1}},b}, m_{\theta _{t-1}} ),
\end{equation*}
where the last equality follows by assumption, \ie, $\theta^{(0)}_t = \theta_{t-1}$.
This result, together with Eq.~\ref{eq:intermed1}, proves the proposition.

\xhdr{Proof of Theorem~\ref{thm:global}}
%
%
Let $\Pi_b :=  \set{\pi \in \Pi \given \EE_{\xb}\left[\pi(\xb)\right] \le b}$ and $\Phi_t =L(\pi^*_{m_{\theta_{t}}, b}, m_{\theta_{t}})-L(\pi^*_{m_{\theta^*}, b}, m_{\theta^*})$. Then, note that, for all $\pi \in \Pi_b$, we have that
\begin{equation}
\nabla_{\theta} ^2\, L(\pi,m_{\theta}) = \EE_{\xb, y}\left[ (1-\pi(\xb)) \, \hes \ell(\mac_{\theta}(\xb), y)  \right]
\implies \emin(1-b) \II \sdleq \nabla_{\theta} ^2\, L(\pi,m_{\theta}) \sdleq \emax \II \label{eq:cx}
\end{equation}
Moreover, we also have that
\begin{align}
 \Phi_{t+1} = L( \pi^* _{m_{\theta_{t+1}},b},m_{\theta_{t+1}})-L( \pi^* _{m_{\theta^*},b},m_{\theta^*})  
& \overset{(i)}{\le}  L( \pi^* _{m_{\theta^*},b} ,m_{\theta_{t+1}})-L( \pi^* _{m_{\theta^*},b},m_{\theta^*})\nn\\
& \overset{(ii)}{\le} \frac{\emax}{2} \bnm{\theta_{t+1}-\theta^*}^2, \label{eq:phi1}
\end{align}
where (i) follows from the fact that $\argmin_{\pi\in\Pi_b} L(\pi,m_{\theta_{t+1}}) = \pi^* _{m_{\theta_{t+1}},b}$ 
and (ii) follows from the Taylor series expansion of $L(\pi^* _{m_{\theta^*},b} ,m_{\theta})$ around $\theta = \theta^*$ and Eq.~\ref{eq:cx}, \ie,
\begin{equation*}
    L( \pi^* _{m_{\theta^*},b} ,m_{\theta_{t+1}}) \le L( \pi^* _{m_{\theta^*},b} ,m_{\theta^*}) +
    \underbrace{\nabla_{\theta} L( \pi^*_{m_{\theta^*},b} ,m_{\theta}) ^\top \big|_{\theta=\theta^{*}} }_{=0} (\theta_{t+1}-\theta^*)   +  \frac{\Lambda_{\max}}{2} \bnm{\theta_{t+1}-\theta^*}^2.
\end{equation*}
Next, we have that
\begin{align}
 L( \pi^* _{m_{\theta_{t}},b},m_{\theta_{t}}) 
 &\overset{(i)}{\ge} L( \pi^* _{m_{\theta_{t}},b},m_{\theta_{t+1}}) + \frac{(1-b)\emin}{2} \bnm{\theta_{t+1}-\theta_t}^2 \nn\\
 & \overset{(ii)}{\ge} L( \pi^* _{m_{\theta_{t+1}},b},m_{\theta_{t+1}}) + \frac{(1-b)\emin}{2} \bnm{\theta_{t+1}-\theta_t}^2, \label{eq:delphi}
\end{align}
where (i) follows from the fact that $\theta_{t+1} = \argmin_{\theta} L(\pi^{*}_{m_{\theta_t}}, m_{\theta})$, the Taylor series expansion of $L( \pi^* _{m_{\theta_t},b} ,m_{\theta})$ around $\theta = \theta_{t+1}$, and Eq.~\ref{eq:cx}, \ie,
\begin{align}
    L( \pi^* _{m_{\theta_t},b} ,m_{\theta_{t}}) \ge L( \pi^* _{m_{\theta_t},b} ,m_{\theta_{t+1}}) & + \underbrace{\nabla_{\theta} L( \pi^* _{m_{\theta_t},b} ,m_{\theta})^\top \big|_{\theta=\theta_{t+1}}}_{=0} (\theta_{t}-\theta_{t+1})\nn\\ 
    & +  \frac{\emin (1-b)}{2} \bnm{\theta_t-\theta_{t+1}}^2, \nn
\end{align}
and (ii) follows from the fact that $\argmin_{\pi\in\Pi_b} L( \pi ,m_{\theta_{t+1}}) = \pi^* _{m_{\theta_{t+1}},b}$.
%
Then, from Eq.~\ref{eq:delphi}, it readily follows that
\begin{align}
 \Phi_t-\Phi_{t+1} \ge \frac{(1-b)\emin}{2} \bnm{\theta_{t+1}-\theta_t}^2. \label{eq:phidiff}
\end{align}
Now, combining Eq.~\ref{eq:phi1} and Eq.~\ref{eq:phidiff}, we have that
\begin{align}
 \Phi_{t+1} & \leq \frac{\emax}{2} \bnm{\theta_{t+1}-\theta^*}^2 \le \emax \left[ \bnm{\theta_{t+1}-\theta_t}^2 +\bnm{\theta_{t}-\theta^*}^2 \right]\nn\\
 & \overset{(i)}{\le}  \frac{2\emax}{(1-b)\emin} (\Phi_t-\Phi_{t+1}) + \frac{4H^2 \emax}{\emin^{2} (1-b)^{2}}. \label{eq:phifinal}
\end{align}
where we have used Proposition~\ref{prop:theta-bound} 
in (i).
Finally, from Eq.~\ref{eq:phifinal}, it readily follows 
that
\begin{equation*}
    \lim_{t \rightarrow \infty} \Phi_{t+1} \leq \frac{4H^2 \emax}{\emin^{2} (1-b)^{2}}.
\end{equation*}
%
This concludes the proof.

\begin{proposition}\label{prop:theta-bound}
Let $\ell(\cdot)$ be convex with respect to $\theta$ and thus the output of the SGD algorithm $\theta_t = \argmin_{\theta} L(\pi^{*}_{m_{\theta_{t-1}}}, m_{\theta})$. Moreover, assume that  $\hes \ell(m_{\theta}(\xb),y) \sdgeq \emin$, with $\emin >0$, and $\ell(\dot)$ be $H$-Lipschitz, \ie, $\ell(m_{\theta}(\xb),y)-\ell(m_{\theta'}(\xb),y) \le H\cdot\bnm{\theta-\theta'}$. Then, we have 
$\bnm{\theta_{t}-\theta^*} \le \frac{2H}{\emin(1-b)}$.
 \end{proposition}
\begin{proof}
We have that
\begin{equation}
 \frac{\emin(1-b)}{2} \bnm{\theta_{t}-\theta^*}^2
\overset{(i)}{\le} L( \pi^* _{m_{\theta^*},b} ,m_{\theta_{t}})-L( \pi^* _{m_{\theta^*},b},m_{\theta^*})
\le H \bnm{\theta_{t}-\theta^*} \label{eq:eeq},
\end{equation}
where 
(i) follows from the Taylor series expansion of $L(\pi^* _{m_{\theta^*,b}},m_{\theta})$ around $\theta=\theta^*$.
Then, from Eq.~\ref{eq:eeq}, it readily follows
\begin{align}
\bnm{\theta_{t}-\theta^*} \le \frac{2H}{\emin(1-b)} 
\end{align}
\end{proof}

\vspace{-7mm}
\section{Scalability Analysis of Algorithm~\ref{alg:sgd}} \label{app:sgd}
In comparison with vanilla SGD, our algorithm just needs to additionally call the function $\textsc{Triage}$ before each iteration. This function first sorts the samples in the corresponding
minibatch in decreasing order of the model loss minus the human loss and then returns the first $\max(\lceil (1-b)|\Dcal|\rceil, p)$ samples. 
Overall, this adds $O(T |\Dcal| \log B)$ to the overall complexity of the training procedure with respect to vanilla SGD, where $B$ is the size of the minibatch used during training, $\Dcal$ is 
the training dataset, and $T$ is the number of steps. 
%
%
Furthermore, note that the function $\textsc{Approxi\-ma\-teTria\-ge\-Po\-li\-cy}$ is called only once and use SGD to train the approximate triage policy of the last predictive
model. Therefore, it does not increase the computational complexity of the overall algorithm. 

\section{Additional Details About the Experiments on Real Data}
\label{app:exp-details}
%
In what follows, we provide additional details regarding the implementation of our method as well as the baselines for the experiments on
real data: 

\begin{itemize}[leftmargin=0.8cm]

\item[---] {Our method: During training, it runs Algorithm~\ref{alg:sgd}.
During test, it lets the humans predict any sample for which $\hat{\pi}_{\gamma}(\xb) \geq \hat{p}_b$, where the threshold $\hat{p}_b$ is found using cross validation.
%
}

\item[---] Confidence-based triage~\citep{bansal2020optimizing}: 
During training, it first estimates the probability $P(h = y)$ that humans predict the true label. Then, it proceeds sequentially and, at each step $t$, it uses SGD to train a 
predictive model $m_{\theta_t}$. 
However, in each iteration of SGD, it only uses the $\min(\lfloor b|\Dcal| \rfloor,n_c)$ training samples with the lowest value of $P(h=y) - \max_{y' \in\Ycal} P(m_{\theta}(\xb)=y')$
in the corresponding mini batch, where $n_c$ is the number of training samples in the mini batch for which $P(h=y)>\max_{y' \in\Ycal} P(m_{\theta}(\xb)=y')$.
During test, it first sorts all the samples in increasing order of $\max_{y' \in\Ycal} P(m_{\theta}(\xb)=y')$ and then lets the humans
predict the first $\min(\lfloor b|\Dcal| \rfloor,n_c)$ samples\footnote{\scriptsize Here, note that the method assumes that the humans are uniformly accurate 
across samples, \ie, $P(h = y \given \xb) = P(h = y)$, both during training and test.}, where $n_c$ is the number of test samples for which $P(h=y)>\max_{y' \in\Ycal} P(m_{\theta}(\xb)=y')$.


\item[---] Score-based triage~\citep{raghu2019algorithmic}: 
During training, it uses SGD to train a predictive model $m_{\theta}$ using all the training samples.
During test, it first sorts all the samples in increasing order of $\max_{y' \in\Ycal} P(m_{\theta}(\xb)=y')$ and then lets the humans
predict the first $\lfloor b|\Dcal| \rfloor$ samples.
Here, note that the method always lets the humans predict $\lfloor b|\Dcal| \rfloor$ samples because its triage policy does not depend
on the human loss.
%

\item[---] Surrogate-based triage~\citep{sontag2020}: 
During training, it trains a predictive model $m_{\theta}$, where $\pi(\xb) = 1$ is just an extra label value $y_{\text{defer}}$, by minimizing a surrogate of the true loss function 
defined in Eq.~\ref{eq:joint-loss}. 
During test, it first sorts all the samples in increasing order of $\max_{y' \in \Ycal} P(m_{\theta}(\xb) = y') - P(m_{\theta}(\xb) = y_{\text{defer}})$ and then lets the human
predict the first $min(\lfloor b|\Dcal| \rfloor,n_c)$ samples where $n_c$ is the number of test samples for which 
$P(m_{\theta}(\xb) = y_{\text{defer}}) >\max_{y' \in\Ycal} P(m_{\theta}(\xb)=y')$.

%
%
%

\item[---] Full automation triage: During training, it uses SGD to both train a predictive model $m_{\theta}$ using all training samples and an approximate triage 
policy $\hat{\pi}_{\gamma}$ that approximates the optimal triage policy $\pi^{*}_{m_{\theta}, b}$.
During test, it lets the humans predict any sample for which $\hat{\pi}_{\gamma}(\xb) \geq \hat{p}_b$, where the threshold $\hat{p}_b$ is found using cross validation.
\end{itemize}

%
In our experiments, our method and all the baselines use the hypothesis class of parameterized predictive models $\Mcal(\Theta)$ parameterized by softmax distributions, \ie,
\begin{align}
 m_{\theta}(\xb) \sim  p_{\theta; \xb} =\text{Multinomial}\left(\left[\exp\left(\phi_{y, \theta}(\xb)\right)\right]_{y\in\Ycal}\right),\nn\\[-4ex]\nn
\end{align}
where, for the nonlinearities $\phi_{\bullet, \theta}$, we use the following network architectures:
\begin{itemize}[leftmargin=0.8cm]
\item[---] Hatespeech dataset: we use the convolutional neural network (CNN) developed by~\cite{kim2014convolutional} for text classification, which consists 
of $3$ convolutional layers with filter sizes $\{3,4,5\}$, respectively and with $300$ neurons per layer. Moreover, each layer is followed by a ReLU non-linearity and 
a max pooling layer. 

\item[---] Galaxy Zoo: we use the deep residual network developed by~\cite{he2015deep}. To this end, we first downsample each RGB channel of each of the
images to size $224\times224$ and standardize its values\footnote{\scriptsize \url{https://pytorch.org/hub/pytorch_vision_resnet/}}. The wide residual network consists of $50$ convolutional 
layers. The first layer is a $7\times7$ convolutional layer followed by a $3\times3$ max pooling layer. The next $48$ convolutional layers have filter sizes of $1\times1$ or 
$3\times3$ which are followed by an average pooling layer. The last layer is a fully connected layer. Each convolution layer is followed by ReLU nonlinearity.
\end{itemize}
%
%
In our method and all the baselines except surrogate-based triage, we use the cross-entropy loss and implement SGD using Adam optimizer~\citep{kingma2017adam} 
with initial learning rate set by cross validation independently for each method and level of triage $b$. 
In surrogate-based triage, we use the loss and optimization method used by the authors in their public implementation.
%
%
Moreover, we use early stopping with the patience parameter $e_p=10$, \ie, we stop the training process if no reduction of cross entropy loss
is observed on the validation set. 
%
%
%
%
Finally, to avoid that the cross entropy loss $\ell(\hat{y}, y)$ becomes unbounded whenever an instance is assigned to a human expert and all 
human experts predicted the same label for that instance in our dataset, we do add/substract an $\epsilon$ value to the estimated values of the conditional
probabilities $P(h \given \xb)$.

\newpage
\section{Additional Evaluation of the Approximate Triage Policy} \label{app:triage}
Figure~\ref{fig:approximate-triage} shows the ratio of model and human losses for those test samples predicted by the model and test samples predicted by 
the humans, as dictated by the approximate triage policy $\hat{\pi}_{\gamma}$, for different values of the maximum level of triage $b$.
We find several interesting insights.
We observe that the approximate triage policy $\hat{\pi}_{\gamma}$ lets the humans predict those samples whose ratio of model and human losses 
is higher, as one could have expected. 
Moreover, in the Hatespeech dataset, we find that the triage policy lets humans predict (almost) all the samples whenever $b = 1$ ($b = 0.8$), \ie, the budget
constraint in the optimization problem defined by Eq.~\ref{eq:optimization-problem} is active.
This suggests that the humans are more accurate than the predictive model throughout the entire feature space. 
In contrast, in the Galaxy zoo dataset, the triage policy does not rely on the human predictions for all samples for $b = 1$. This suggests that the humans 
are less accurate than the predictive model in some regions of the feature space.
\begin{figure*}[!t]
\centering
\subfloat[Hatespeech]{
\scriptsize
\stackunder[2pt]{\includegraphics[width=0.23\textwidth]{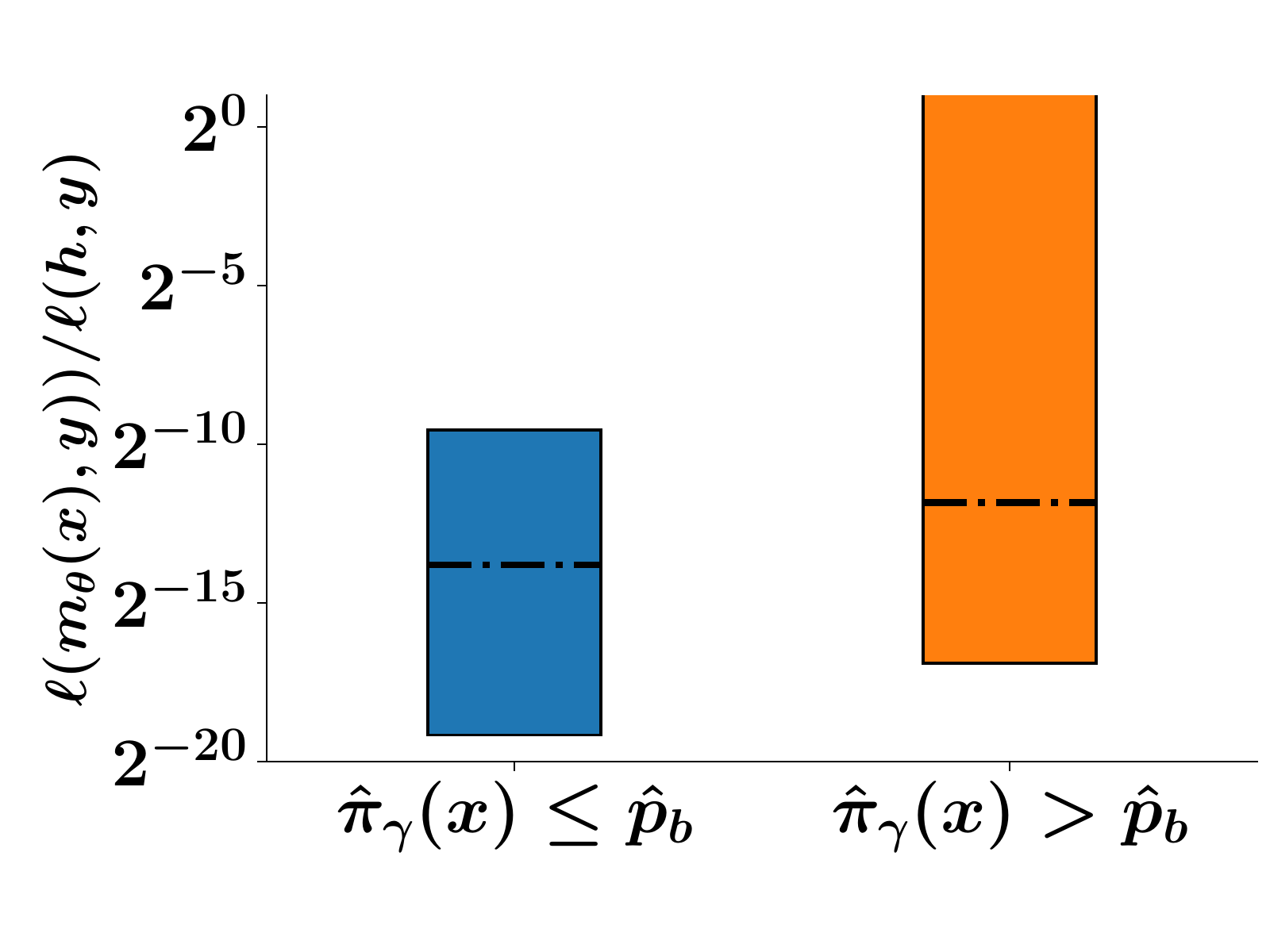}}{$b=0.4, \hat{p}_b=0$} \hspace*{0.2cm}
\stackunder[2pt]{\includegraphics[width=0.23\textwidth]{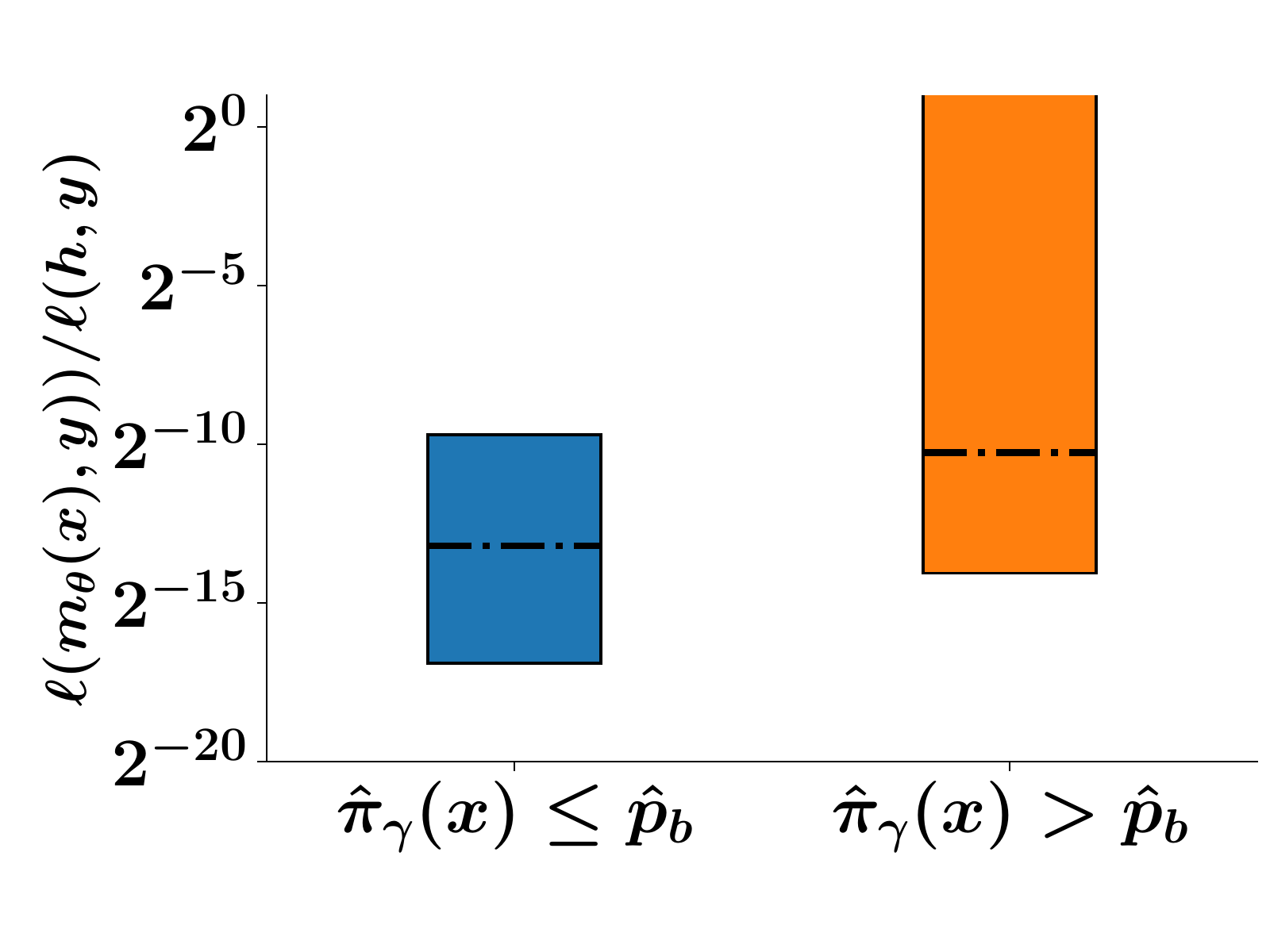}}{$b=0.6, \hat{p}_b=0$} \hspace*{0.2cm}
\stackunder[2pt]{\includegraphics[width=0.23\textwidth]{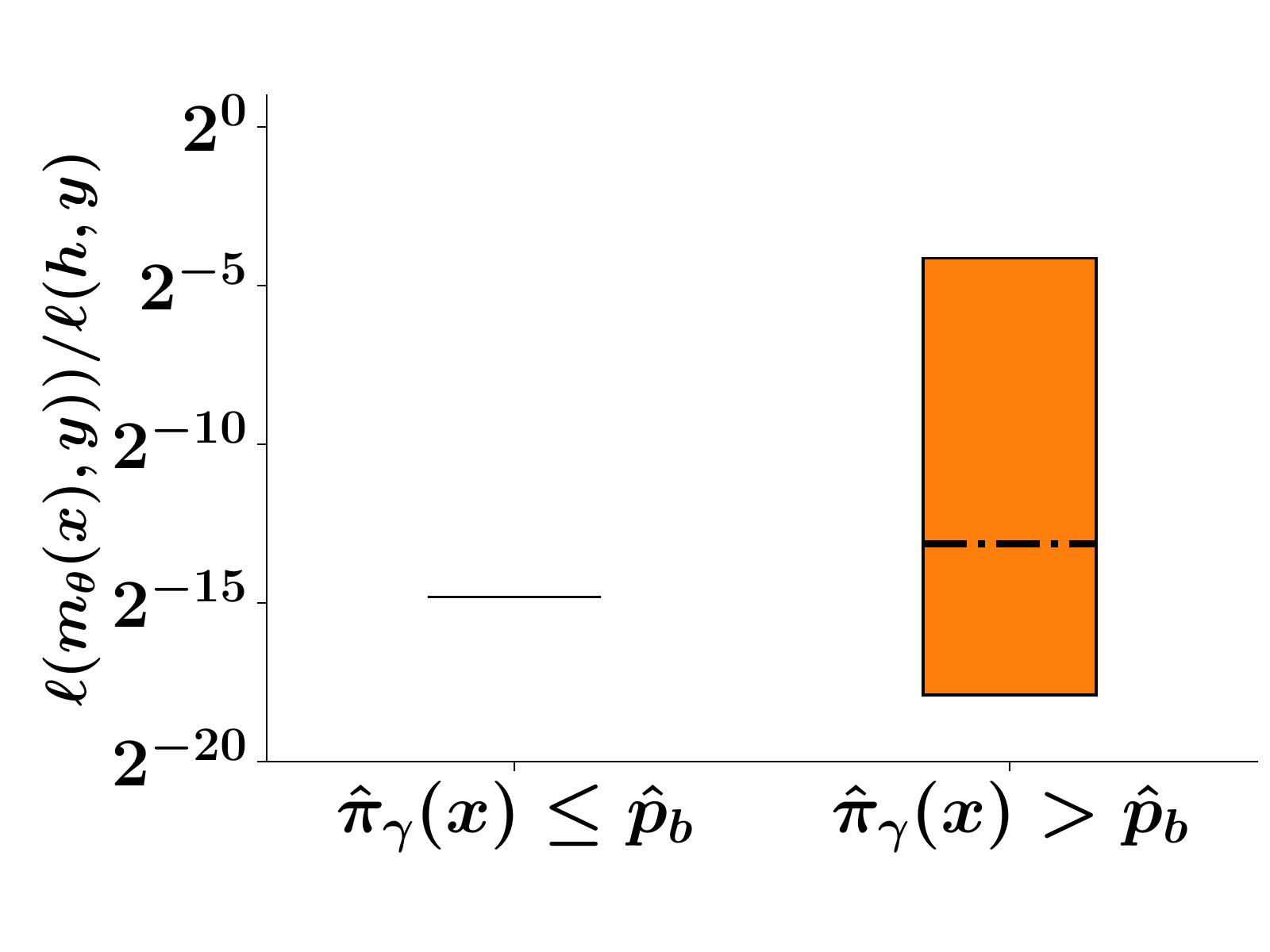}}{$b=0.8, \hat{p}_b=0$} \hspace*{0.2cm}
\stackunder[2pt]{\includegraphics[width=0.23\textwidth]{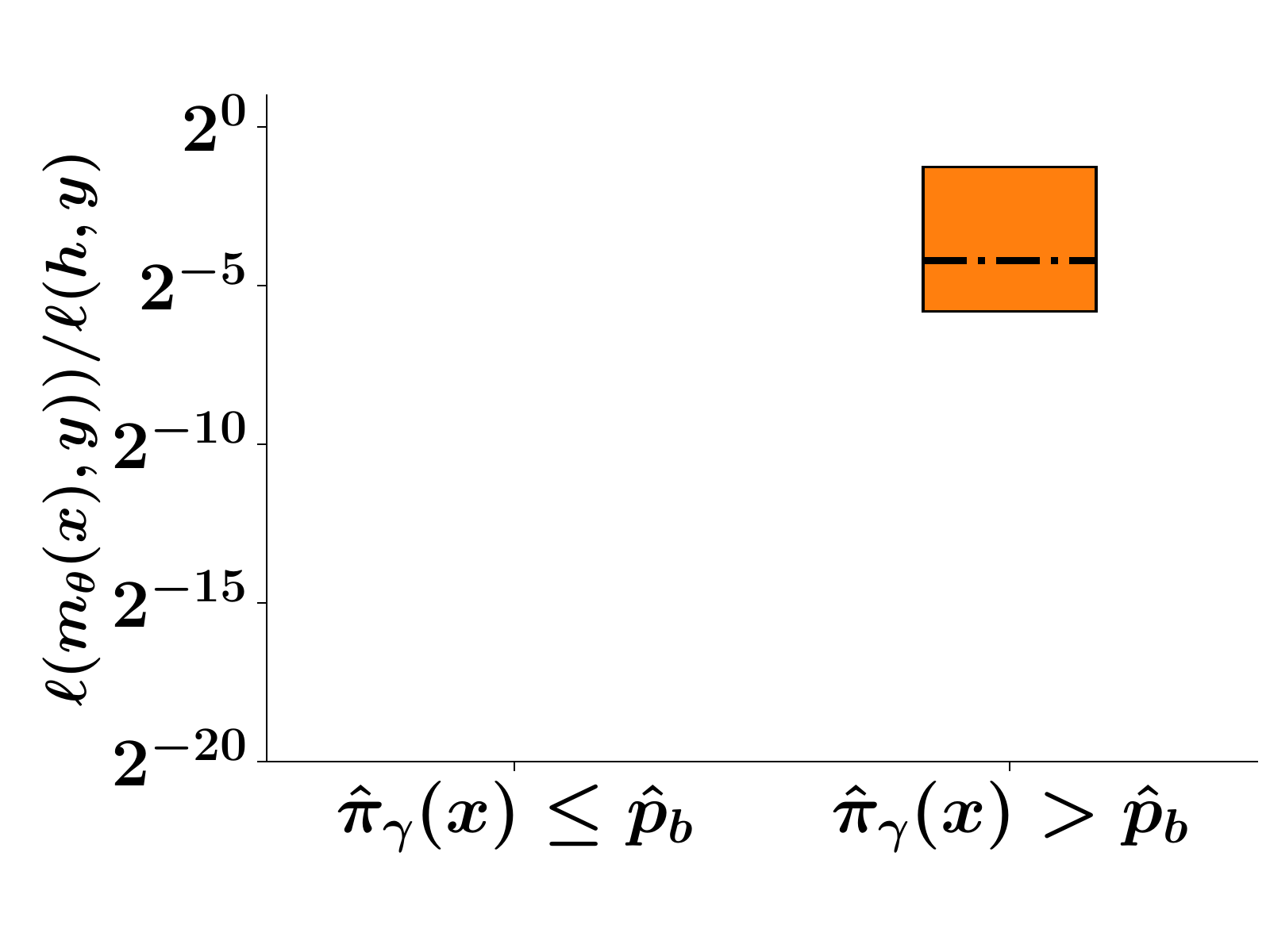}}{$b=1.0, \hat{p}_b=0$}
} \\
\subfloat[Galaxy zoo]{
\scriptsize
\stackunder[2pt]{\includegraphics[width=0.23\textwidth]{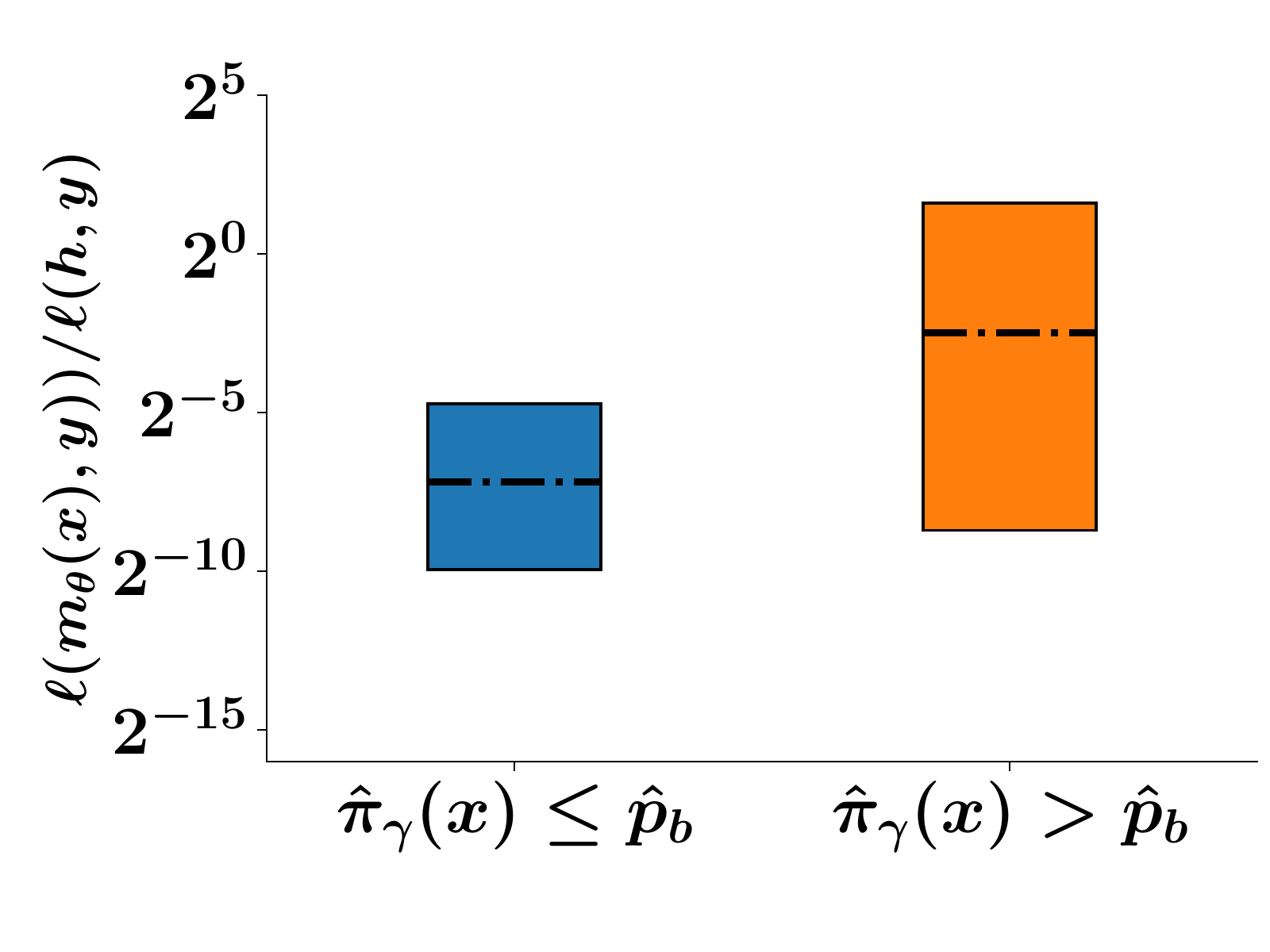}}{$b=0.4, \hat{p}_b=0.25$} \hspace*{0.2cm}
\stackunder[2pt]{\includegraphics[width=0.23\textwidth]{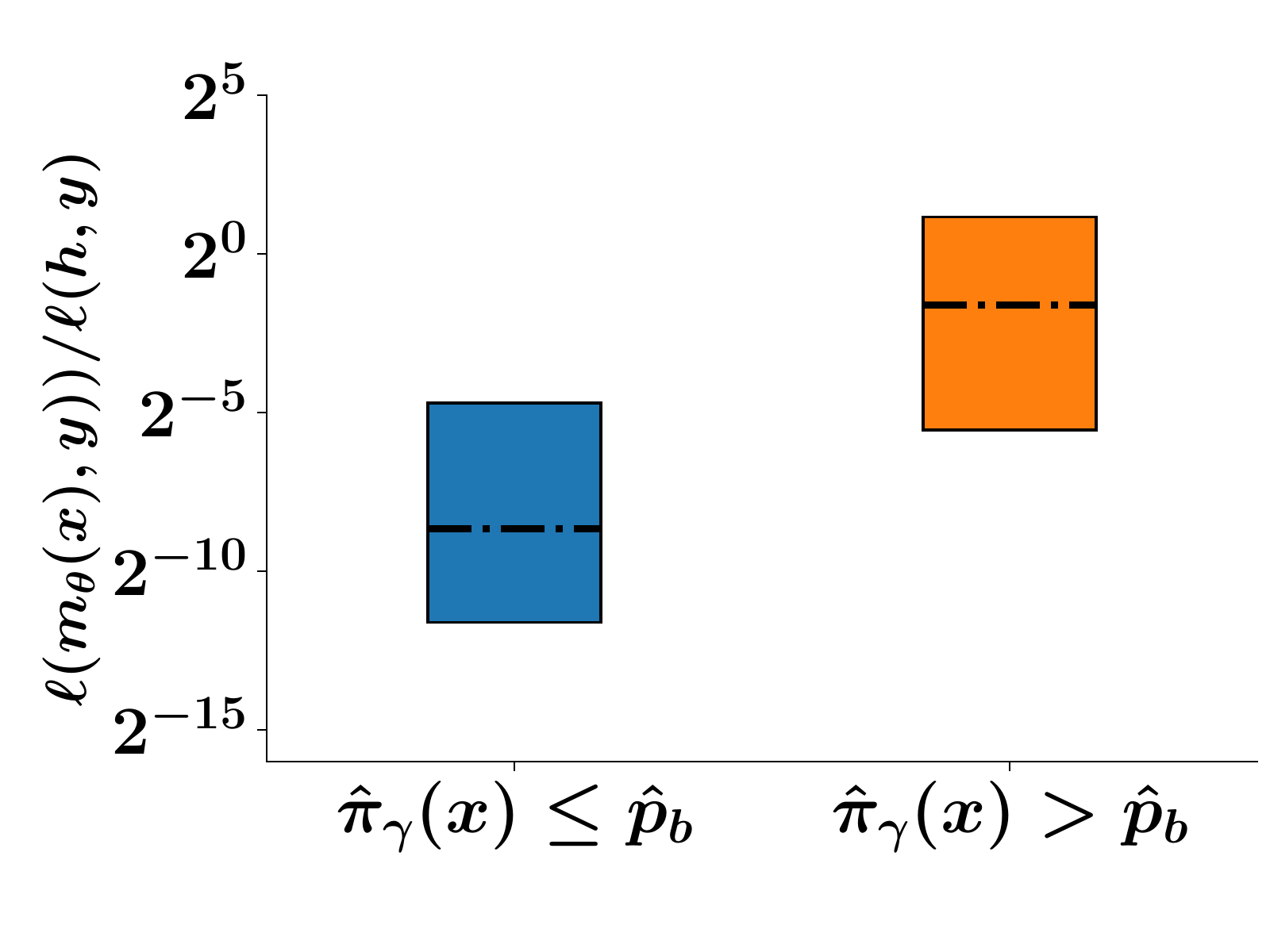}}{$b=0.6, \hat{p}_b=0.21$} \hspace*{0.2cm}
\stackunder[2pt]{\includegraphics[width=0.23\textwidth]{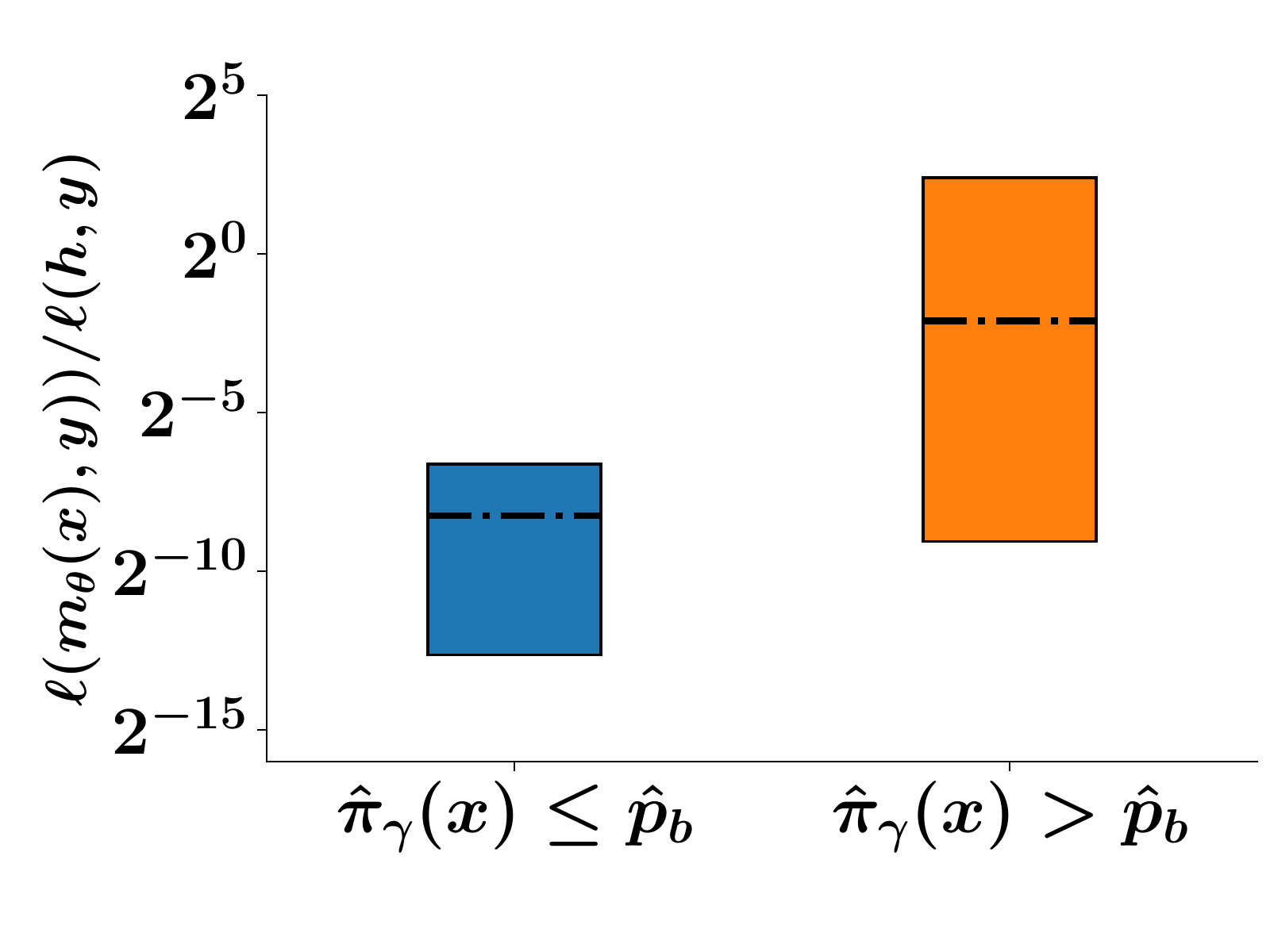}}{$b=0.8, \hat{p}_b=0.21$} \hspace*{0.2cm}
\stackunder[2pt]{\includegraphics[width=0.23\textwidth]{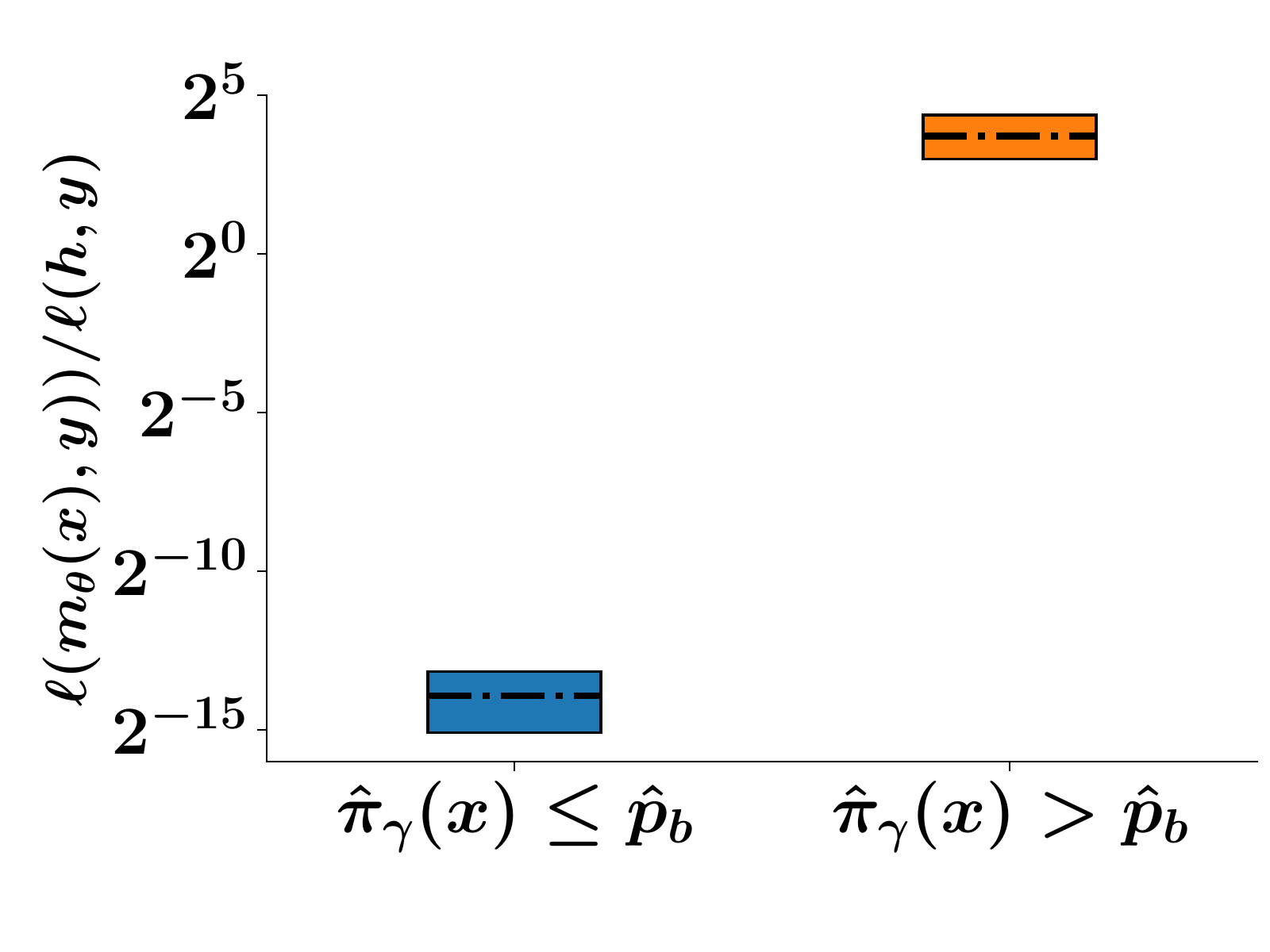}}{$b=1.0, \hat{p}_b=0.37$}
}
\caption{Ratio of model and human losses for test samples predicted by the model and test samples predicted by the humans, as dictated by the 
approximate triage policy $\hat{\pi}_{\gamma}$, for different values of the maximum level of triage $b$. In each panel, the threshold $\hat{p}_b$ is found using 
cross validation. Boxes indicate 25\% and 75\% quantiles and the horizontal lines indicate median values.}
\label{fig:approximate-triage}
\end{figure*}

\end{document}